\documentclass{article}

\PassOptionsToPackage{numbers, compress}{natbib}

 \usepackage[main, final]{neurips_2026}

\usepackage[utf8]{inputenc} 
\usepackage[T1]{fontenc}    
\usepackage{hyperref}       
\usepackage{url}            
\usepackage{amsfonts}       
\usepackage{nicefrac}       
\usepackage{microtype}      
\usepackage[dvipsnames]{xcolor}

\usepackage{graphicx}
\usepackage{subfigure}
\usepackage{{booktabs, xcolor, colortbl}}
\usepackage{rotating}
\usepackage{xcolor}
\usepackage{colortbl}
\usepackage{amsmath, amssymb}
\usepackage{bm}
\usepackage[ruled,vlined,linesnumbered]{algorithm2e}
\usepackage{xcolor}
\usepackage{amsmath, amssymb}
\usepackage{algpseudocode}
\usepackage{bm}
\usepackage{multirow}
\usepackage{xcolor}
\usepackage{hhline}
\usepackage{multirow}
\usepackage{rotating}
\usepackage{amsmath}
\usepackage{makecell}
\usepackage{booktabs}
\usepackage[capitalize,noabbrev]{cleveref}
\usepackage{aliascnt}
\usepackage{amssymb}
\usepackage{mathtools}
\usepackage{amsthm}
\usepackage{diagbox}
\usepackage{tcolorbox}
\usepackage{amsmath, amssymb, amsthm}
\usepackage{bm}                        
\usepackage{mathtools}                 
\usepackage{calc}
 \newlength{\InputIndent}
\usepackage{xcolor}
\usepackage{booktabs}
\usepackage{hyperref}
 \usepackage{wrapfig}

\newcommand{\R}{\mathbb{R}}

\newcommand{\rot}[1]{\rotatebox{90}{#1}}
\renewcommand\tabcolsep{5pt}
\renewcommand\arraystretch{1.2}

\newcommand{\up}[1]{\textcolor{green!60!black}{\scriptsize$\uparrow$#1}}
\newcommand{\dn}[1]{\textcolor{red}{\scriptsize$\downarrow$#1}}

\usepackage{makecell}
\usepackage{enumitem}
\usepackage[table]{xcolor}
\usepackage{makecell}

\definecolor{clrOwn}{RGB}{83,74,183}
\definecolor{clrRoute}{RGB}{24,95,165}
\definecolor{clrCtrl}{RGB}{133,79,11}
\definecolor{clrOut}{RGB}{15,110,86}
\definecolor{clrGate}{RGB}{153,60,29}
\definecolor{clrComment}{RGB}{29,158,117}

\definecolor{moccolor}{HTML}{FFF3E6} 
\theoremstyle{plain}

\newtheorem{theorem}{Theorem}[section]

\newaliascnt{proposition}{theorem}
\newtheorem{proposition}[proposition]{Proposition}
\aliascntresetthe{proposition}

\newaliascnt{lemma}{theorem}

\aliascntresetthe{lemma}

\newaliascnt{corollary}{theorem}
\newtheorem{corollary}[corollary]{Corollary}
\aliascntresetthe{corollary}

\theoremstyle{definition}
\newaliascnt{definition}{theorem}

\aliascntresetthe{definition}

\newaliascnt{assumption}{theorem}

\aliascntresetthe{assumption}

\theoremstyle{remark}
\newaliascnt{remark}{theorem}

\aliascntresetthe{remark}

\newtcbox{\hl}[1][]{
  on line, 
  arc=1pt,                     
  colback=green!15,            
  colframe=green!60!black,     
  boxrule=0.8pt,               
  boxsep=1pt,                  
  left=2pt, right=2pt, top=2pt, bottom=2pt,
  #1
}

\title{Mixture-of-Control: State-Aware Fine-Tuning for Transformer-based Models}

\author{%
  Duc-Anh Nguyen\textsuperscript{1,2,}%
  \thanks{Equal contribution.}%
  \textsuperscript{\;\,}%
  \thanks{This work was done during Duc-Anh Nguyen's internship at Qualcomm.}%
  \textsuperscript{\;\,}%
  \thanks{Corresponding authors: \texttt{nguyenducanh909.bkhn@gmail.com}, \texttt{toatran@qti.qualcomm.com}.}
  \quad
  Tien Ngoc Luu\textsuperscript{1,}\footnotemark[1]
  \quad
  Tung Pham\textsuperscript{1}
  \quad
  Toan Tran\textsuperscript{1,}\footnotemark[3]
  \\[0.3cm]
  \NoHyper
  \textsuperscript{1}Qualcomm AI Research\thanks{Qualcomm AI Research is an initiative of Qualcomm Technologies, Inc.}
  \quad
  \textsuperscript{2}University of Virginia
  \endNoHyper
}

\begin{document}

\maketitle

\begin{abstract}
State-based fine-tuning has emerged as a compelling alternative to weight-based adaptation for transformers, updating lightweight controls into states rather than model weights, offering substantial memory savings while retaining parameter efficiency. However, most existing state-based methods typically apply only per-block control updates, which limits inter-block information exchange and restricts representational adaptation. Meanwhile, prior mechanisms that enable cross-block communication often introduce considerable computational overhead, reducing their practicality for efficient fine-tuning. We introduce Mixture-of-Control (MoC), an efficient cross-layer communication that adaptively integrates local and global control signals to enhance representation learning during adaptation. MoC treats block-wise control states as experts in a sparse mixture-of-experts process, enabling efficient communication across transformer blocks. Experiments on 33 datasets and 10 pretrained models across diverse transformer-based benchmarks demonstrate that MoC consistently outperforms existing state-based fine-tuning methods while maintaining comparable efficiency in memory and compute. 

\end{abstract}

\section{Introduction}

Fine-tuning large pretrained models is a standard approach to adapt foundation models to downstream tasks, but full fine-tuning is expensive: it updates and stores billions of parameters per task. Parameter-efficient fine-tuning (PEFT) reduces this cost by freezing the backbone and learning only a small set of task parameters. Among PEFT methods, LoRA~\cite{hu2022lora} is particularly effective, replacing weight updates with a low-rank decomposition that largely preserves performance while sharply reducing trainable parameters.

Many LoRA variants--including VeRA~\cite{kopiczko2023vera}, BinaryLoRA~\cite{zhang2024optimal}, MoLE~\cite{wu2024mixture}, and DoRA~\cite{liu2024dora}--push parameter or compute efficiency via vectorized updates, binary adaptation, expert mixtures, or decomposed updates. However, most remain layer-local, offering limited cross-layer communication and constraining structured representation adaptation. MoLEx~\cite{teo2025molex} partially addresses this with a Sparse Mixture-of-Layer-Experts, but its Top-$K$ selection requires extra forward passes, raising per-layer cost, especially in large models.

Beyond weak cross-layer coordination, many methods optimize parameter count but overlook training-time memory. For example, halving DoRA's rank reduces GPU memory by only 0.256GB on an 8B model~\cite{zhangweight}, showing that fewer trainable parameters need not yield meaningful memory savings. Similarly, existing cross-layer interaction mechanisms such as MoLEx~\cite{teo2025molex} often add substantial compute and memory overhead, limiting scalability.

These two issues--restricted cross-layer communication and high memory overhead--motivate a different perspective: fine-tuning via control theory~\cite{zhang2024parameter}. Rather than modifying system parameters, control systems steer states via feedback. Parallel Control~\cite{zhangweight} casts LoRA as a control term and activations as states, reducing memory by avoiding intermediate-state storage. Yet it remains local to each layer and lacks global coordination across depth. This motivates our central research question:

\emph{
``How can state‑based fine‑tuning be designed to capture deep cross‑layer interactions without sacrificing its core advantage: memory efficiency?''
}

We answer this with \textbf{Mixture-of-Control (MoC)}, a lightweight state-based method that couples local and global signals via a recurrent mixture-of-experts control pathway. MoC treats each layer’s low-rank control module as an \emph{expert} and uses a shared gate to sparsely route each layer’s representation to a Top-$K$ set of experts, producing a global control signal. This global signal is then mixed with the layer-local control (e.g., $\alpha$ vs.\ $1{-}\alpha$) and injected into the frozen block, enabling adaptive information propagation across depth. Our MoC supports cross-layer communication via low-rank projections without extra pretrained-block passes, preserving state-memory efficiency with minimal overhead. It is architecture-agnostic (FFN or attention; shared or separate gates) and improves robustness and adaptation across benchmarks while matching PEFT-level memory use.

In summary, our contributions are fourfold:
\begin{itemize} 
    \item We identify key limits of existing PEFT and state-based adaptation: controllers are mostly block-local (weak global coordination), while cross-block mechanisms often add substantial, block-dependent compute that impedes scaling to deep/complex architectures. 
    \item We show that cross-block communication can be realized via a lightweight low-rank \emph{control} pathway that exchanges information across layers without reusing pretrained blocks, offering an efficient alternative to full-block interaction.
    \item We propose \textbf{Mixture-of-Control (MoC)}, an efficient cross-layer communication framework that integrates block-local control with a sparsely routed global mixture of control experts to enhance the representation learning, supported by theory showing expressiveness preservation and improved stability/gradient flow/representation error under mild assumptions.
    \item We conduct extensive experiments across 33 benchmarks spanning NLU, NLG, and vision tasks on 10 pretrained Transformer backbones, showing that MoC consistently improves adaptation while matching strong PEFT baselines in parameter and memory efficiency.
\end{itemize}

\section{Related Works}

\paragraph{Parameter-Efficient Fine-Tuning (PEFT)}
The rapid growth of Transformer-based models \cite{vaswani2017attention, Alexey2020_ViT} has made full fine-tuning costly, motivating PEFT. Early PEFT methods tune selected layers \cite{6909618, alain2018understandingintermediatelayersusing, lee2023surgicalfinetuningimprovesadaptation, kaplun2023less}, while later approaches add lightweight trainable components. Adapter-based methods insert bottleneck modules \cite{houlsby2019parameterefficienttransferlearningnlp, mahabadi2021compacterefficientlowrankhypercomplex}, and prompt-based methods optimize soft prompts \cite{lester2021powerscaleparameterefficientprompt, liu-etal-2022-p, li2021prefixtuningoptimizingcontinuousprompts}, though they are often sensitive to initialization \cite{razdaibiedina2023progressivepromptscontinuallearning}.

\paragraph{Low-Rank Adaptation (LoRA)}\cite{hu2022lora} reduces the number of trainable parameters by modeling weight updates with low-rank matrices. Subsequent variants further improve efficiency through parameter sharing, quantization, or adaptive rank/parameter selection \cite{kopiczko2023vera, dettmers2023qloraefficientfinetuningquantized, hyeonwoo2023fedparalowrankhadamardproduct,mai2025lessons}, while extensions such as MoLE and DoRA (which separates magnitude and direction) further enhance adaptability and performance \cite{wu2024mixture, liu2024dora,jie2023fact, xiao2026not,liu2026flexlora}. However, most methods remain layer-local, thereby limiting cross-layer information exchange during adaptation.

\paragraph{Cross-Layer Communication for Adaptation}
To overcome the layer-local limitations of LoRA-based methods, Mixture of Layer Experts (MoLEx) \cite{teo2025molex} treats each pretrained layer as an expert and performs TopK routing across depth to aggregate cross-layer information. Although effective, this strategy requires additional forward passes through selected layers, which incurs substantial computational overhead, particularly in deep or complex architectures.

\paragraph{Control-Theoretic View and State-Based Tuning}
PEFT can be viewed through control theory
\cite{kirk2004optimal}, treating deep networks as dynamical systems evolving feature states across layers \cite{li2018maximumprinciplebasedalgorithms}. This view motivates work on continuous-depth models, stability, and optimization \cite{tabuada2024universalapproximationpowerdeep, lu2020finitelayerneuralnetworks, nguyen2024pidformertransformermeetscontrol, benning2019deeplearningoptimalcontrol}. Zhang et al. \cite{zhang2024parameter} cast LoRA as a control input over nominal pretrained dynamics, and \cite{zhangweight} introduce Parallel Control to reduce memory by skipping intermediate states. 
We extend this framework with a scalable state-based adaptation method that improves performance while preserving similar memory efficiency.

\section{Preliminaries}
\label{sec:preliminaries}
\subsection{Weight-based Fine-Tuning}
LoRA~\cite{hu2022lora} fine-tunes $W_0 \in \mathbb{R}^{d\times k}$ by adding a low-rank update $\Delta W = BA$ with $B \in \mathbb{R}^{d\times r}$, $A \in \mathbb{R}^{r\times k}$, and $r \ll \min\{d,k\}$, yielding $W' = W_0 + BA$ where $W_0$ is frozen and $A,B$ are trainable. Many variants exploit this structure for efficiency, including VeRA \cite{kopiczko2023vera}, MoLE \cite{wu2024mixture}, and DoRA \cite{liu2024dora}, which leverage low-rank structures to reduce computational cost. These mechanisms are categorized as weight-based fine-tuning (Weight-FT), which adapts model weights directly. Its effectiveness is often attributed to the low intrinsic-rank structure of task-specific updates \cite{hu2022lora, aghajanyan2021intrinsic}. Mathematically, let $x_0$ be the input, $x_t$ be the output representation at block $t$, and $a_t$ be a nonlinear activation (ReLU or GeLU), a pretrained transformer block can be adapted as follows \cite{chen2022adaptformer}:
\begin{align}
x_{t+1} = x_t + a_t\!\left((W_t + \Delta W_t)x_t\right).
\end{align}
Weight-FT often incurs high memory costs due to storing intermediate states. To mitigate this, the State-Based Framework \cite{zhangweight} reformulates fine-tuning from a control-theoretic perspective.

\subsection{State-based Fine-Tuning}

\paragraph{Forward Pass as the Feedback Control}
Consider the Ordinary Differential Equation system \cite{franklin2002feedback}:
\begin{equation}\label{eq:feedback-control}
\dot{x}(t) = f_t(W(t)x(t)) + G(t)u(t),    
\end{equation}
where $x(t)$ denotes the state, $W(t)$ the state matrix, $f_t$ the intrinsic dynamics, and $G(t)u(t)$ the control term. Using state feedback $u(t)=K(t)x(t)$ with feedback gain $K(t)$, \cref{eq:feedback-control} becomes
\begin{equation}
    \label{continuous_feedback_control}
    \dot{x}(t) = f_t(W(t)x(t)) + G(t)K(t)x(t).
\end{equation}
Discretizing \cref{continuous_feedback_control} yields
\begin{equation}
\label{feedback_control}
    x_{t+1} = x_t + f_t(W_t x_t) + G_t K_t x_t.
\end{equation}
This matches a neural-network forward pass when $x_t$ is the $t$-block state representation.

\paragraph{State-based Fine-Tuning}
Motivated by this analogy, \citet{zhangweight} introduced Parallel Control, interpreting $G_tK_t$ as a low-rank control term, i.e., 
\begin{equation}
x_{t+1} = x_t + f_t(W_t x_t) + g_t(x_t),
\end{equation}
where $f_t$ is the pretrained block (e.g., attention/ FFN), $g_t(x_t)=\Delta W_t x_t$ is implemented by a low-rank adapter. This view retains efficient adaptation while reducing intermediate-memory overhead.

\subsection{Sparse Mixture of Experts (SMoE)}
\label{subsec:SMoE}
A SMoE layer contains $n$ experts and a gating network to process different aspects of the input. Given $m$ token embeddings $\mathbf{X} \in \mathbb{R}^{m \times d}$ and gate weights $\mathbf{W}_{\text{gate}} \in \mathbb{R}^{n \times d}$, the routing scores are $
\mathbf{S} = \mathbf{X}\mathbf{W}_{\text{gate}}^\top \in \R^{m \times n}.$
For each token $\mathbf{x}_i \in \R^d$, only the top-$k$ expert indices $\mathcal{T}_i \subseteq \{1, \dots, n\}$ with $|\mathcal{T}_i| = k$ are activated. Their normalized routing weights are
$w_{i,j} =
\begin{cases}
\displaystyle\frac{\exp(s_{i,j})}{\sum_{j' \in \mathcal{T}_i} \exp(s_{i,j'})} & , \text{ if } j \in \mathcal{T}_i, \\
\qquad \quad 0 &, \text{ otherwise.}
\end{cases}$
Each expert $e_j: \mathbb{R}^d \rightarrow \mathbb{R}^d$ is presented by a feedforward network, then the output is aggregated as
\begin{equation}
\mathbf{y}_i = \sum_{j \in \mathcal{T}_i} w_{i,j} \cdot e_j(\mathbf{x}_i).
\end{equation}
To mitigate routing collapse \cite{fedus2022_switch, lepikhin2020_gshard,nguyen2025selective} (where a few experts dominate token assignments), a commonly used auxiliary load-balancing loss is introduced: $
\mathcal{L}_{\text{Balance}}
= \lambda \sum_{j=1}^{n} \tau_j P_j$, where $\lambda$ is a hyperparameter, with $P_j = \frac{1}{m} \sum_{i=1}^{m} \operatorname{softmax}(\mathbf{s}_{i,:})_j
$ and $\tau_j = \frac{n}{k m} \sum_{i=1}^{m} \mathbf{1}(\text{token } i \text{ selects expert } j)$ .

\section{Methodology}

This section introduces Mixture-of-Control (MoC), a generalized state-based fine-tuning framework for efficient cross-block communication. Unlike methods requiring $>2\times$ forward passes through pretrained blocks, MoC uses lightweight control states to enable information transfer with negligible compute overhead, making it practical for large models and complex architectures.

\subsection{Generalized State-Based Tuning: Retaining MoLEx Representational Power}
\label{sec:4.1}
We begin from the following generalized state-based update:
\begin{align}
\label{general_state_based}
    x_{t+1} = x_t + f_t(W_t x_t) + \alpha \, g_t(x_t) + (1-\alpha) \, g_{h_t}(x_t),
\end{align}
where $W_t$ denotes the pretrained weight of block $t$, $g_t$ is its local low-rank control, $g_{h_t}$ is a control chosen from another block $h_t$, and $\alpha \in [0,1]$ balances local and cross-layer information. Standard state-based tuning is recovered when $h_t=t$, while $h_t\neq t$ enables cross-block communication through lightweight low-rank controls rather than full-block reuse.  

The adoption of low-rank controls is motivated by \cite{zhangweight}, which shows that they efficiently adapt block representations as an alternative to Weight-FT. In constrast to MoLEx~\cite{teo2025molex}, which relies on pretrained block reuse under Weight-FT, our method facilitates lightweight cross-block communication while substantially reducing training time and memory overhead. This design choice is further supported by \cite{zeng2023expressive}, which demonstrates that sufficiently expressive LoRA adapters can approximate arbitrary target functions.

We analyze linear and nonlinear smooth $f_t$. Generalized state-based fine-tuning retains MoLEx expressivity: it is exact for linear $f_t$ (under mild conditions) and bounded-approximate for nonlinear $f_t$. We present the linear proof for $f_t(x)=x$; general linear maps follow directly.

\begin{theorem}[MoLEx Expressivity via State Controls]
\label{prop4.1:Expressive_moc}
Consider the MoLEx update over $L$ blocks defined by
\begin{align}
    z_{t+1}
    &= z_t + \alpha f_t\!\left((W_t+\Delta W_t')z_t\right)
    + (1-\alpha) f_{h_t}\!\left((W_{h_t}+\Delta W_{h_t}')z_t\right),
\end{align}
where $z_t$ is the output of the $t$-th block and the generalized state-based update defined in \cref{general_state_based}. Assuming that $\alpha>\tfrac{1}{2}$ and $z_1=x_1$, then there exist controls $g_k:\mathbb{R}^d\to\mathbb{R}^d$ such that: when $f_k$ is linear we have $z_t=x_t$ for all $t\in\{1,\ldots,L+1\}$; moreover, when $f_k$ is smooth nonlinear under small perturbations, i.e., $\forall t$, $\|\Delta W'_t\|<\xi$ ($\xi>0$ and sufficiently small), then $\mathbb{E}\!\left[\|z_t-x_t\|^2\right]$ is bounded and vanishes as $\xi \to 0$.
\end{theorem}

This theorem indicates that, under mild assumptions, generalized state-based updates can preserve MoLEx-induced representations through state controls. Motivated by this expressivity guarantee, we next develop an adaptive state-based framework that enables inter-block communication.

\begin{figure*}
\vspace{-0.1in}    
    \centering
    \includegraphics[width=1.05\linewidth]{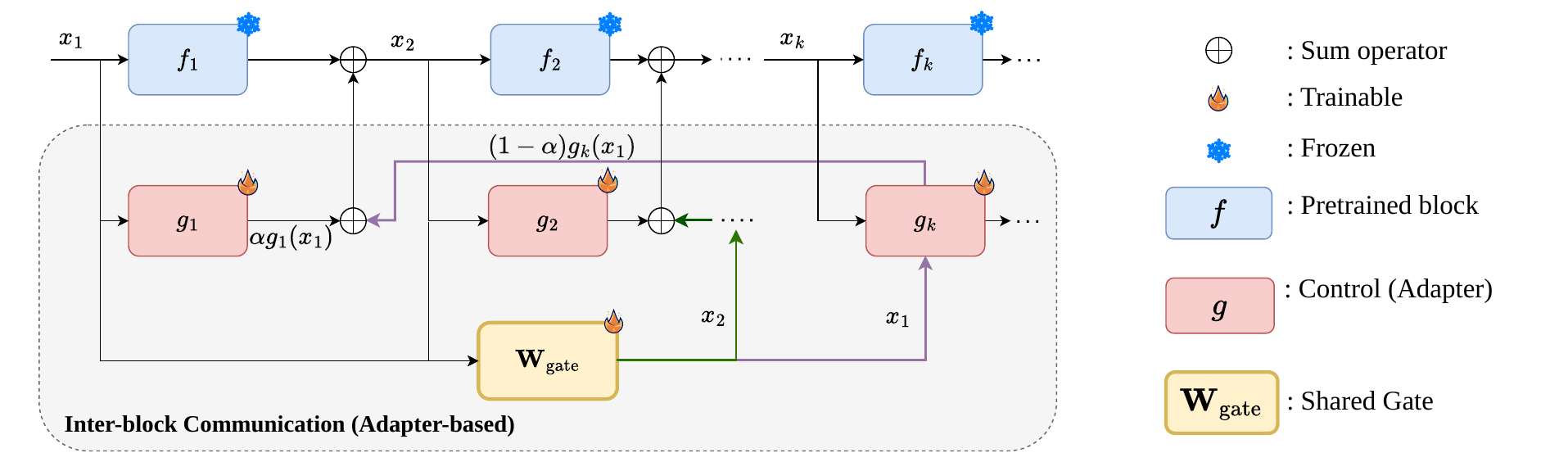}
    \caption{Overview of our Mixture-of-Control (MoC): For an arbitrary input $x$ at the block $i$, the pretrained block $f_i$ and the local control block $g_i$ process the representation in parallel. A shared gate selects the Top-$K$ control block from all layers, and the selected controller is fused with the local one to form the block control signal. The figure illustrates this mechanism using $x_1$ and later $x_2$ as the input under Top-$K$=1.}
    \label{fig:moc_architecture}

\vspace{-0.15in}    
\end{figure*}
\subsection{Mixture-of-Control (MoC) for State-based Fine-Tuning}
\label{sec4.2:moc}

As discussed in \cref{sec:4.1}, inter-block communication can be implemented via low-rank control modules instead of reusable pretrained blocks, substantially reducing computation. We propose Mixture-of-Control (MoC), a lightweight, SMoE-inspired fine-tuning method that captures global information in a depth-adaptive manner. MoC treats the control modules $\{g_k\}_{k=1}^L$ as experts and employs a shared gating matrix $\mathbf{W}_{\text{gate}}$ across layers. At each block, the input selects Top-$K$ experts to produce a global control signal; this signal is then combined with local control and added to the pretrained block output before passing to the subsequent block (see \cref{fig:moc_architecture}).

Formally, MoC coordinates local and global control signals adaptively for efficient inter-block information exchange, instantiating \cref{general_state_based} as:
\begin{equation}\label{moc}
x_{t+1} = x_t + f_t(W_t x_t) + \alpha g_t(x_t)  
+ (1-\alpha)\sum_{k=1}^L \pi_k(x_t)g_k(x_t),
\end{equation}
where 
\begin{equation}\label{eq:routing-weight}
\pi_k(x_t)=\operatorname{softmax}(\operatorname{TopK}(\mathbf{W}_{\text{gate}}x_t))_k
\end{equation}
denotes the routing weight for control expert $k$, $L$ is the number of control sites, and $\mathbf{W}_{\text{gate}}$ is the shared routing matrix. The coefficient $\alpha$ balances local token-specific information with global information injected by routed control experts. As presented in \cref{alg:moc}, MoC adaptively modulates multiple control sites through a shared routing mechanism. In our experiments, we use top-$K$ routing with $K=1$ and linear low-rank control functions $g_k$, with ablations on $K$ and control design in \cref{main:ablation}. To improve training efficiency for Transformer models, we employ \emph{batch-level routing}, in which token-level routing scores are aggregated using either the mean or the mode. During autoregressive inference, each sample is routed independently based solely on its available prefix. Additional routing details are provided in \cref{appendix:exp_details} of Appendix.

MoC extends state-based fine-tuning by combining block-local controls with routed control experts for lightweight inter-block exchange. For Transformers, we study four variants: \textbf{MoC (FFN)} routes controls in FFN blocks, and \textbf{MoC (ATTN)} in attention blocks. For finer control, \textbf{MoC (ATTN+FFN, w/o sharing)} uses separate gates for attention and FFN, while \textbf{MoC (ATTN+FFN)} shares a single gate to coordinate both. Unless noted otherwise, \textbf{MoC} refers to \textbf{MoC (ATTN+FFN)}.

\textit{Sharing a gate across attention and FFN follows fine-grained SMoE:} granular routing improves expert specialization and utilization without additional cost \cite{krajewski2024scaling, lo2025closer}. In MoC, a shared gate jointly coordinates both pathways for finer adaptation at similar overhead. Overall, MoC enables depth-adaptive Top-$K$ information exchange, increases capacity under a frozen backbone, and remains a lightweight, architecture-agnostic drop-in for standard fine-tuning.
    
We rigorously evaluate MoC through a theoretical analysis of \textit{stability}, \textit{gradient flow}, and \textit{representation error}, and compare the results to standard Parallel Control.

\subsection{Theoretical Analysis}
\label{sec4.3:theoretical_analysis}
\paragraph{Stability Analysis.}
We demonstrate that, under the effective routing, MoC exhibits the more stable forward dynamics than the Parallel Control dynamics.

\begin{proposition}[Stability with Top-$K$ Routing]
\label{thm:moc_stability}
Consider a network with $L$ control sites, where each control is a low-rank projection $\Delta W_k = B_kA_k$ with $\|B_k\|\leq \beta_k$ and $\|A_k\|\leq \gamma$. Let the router select a top-$K$ index set $S_K(x_t)$ with weights $\pi_k(x_t)$ given in \cref{eq:routing-weight}. Define the effective routed magnitude as $
\beta_K(x_t)=\sum_{k\in S_K(x_t)}\pi_k(x_t)\beta_k
$. 

Let $\{x_t^c\}_{t=1}^{L+1}$ and $\{x_t\}_{t=1}^{L+1}$ denote the representations produced by Parallel Control and MoC, respectively. Then for each $t, 
\|x_{t+1}^c\|
\leq
(1+L_f+\beta_t\gamma)\|x_t^c\|$,
whereas MoC with mixing coefficient $\alpha\in[0,1]$ satisfies
$\|x_{t+1}\|
\leq
\bigl(1+L_f+\alpha\beta_t\gamma+(1-\alpha)\beta_K(x_t)\gamma\bigr)\|x_t\|,
$ where $L_f$ bounds the Lipschitz contribution of the frozen dynamics. Moreover, if $x_1=x_1^c$ and the router selects experts such that $\mathbb{E}[\beta_K(x_t) | x_t]<\beta_t$, then MoC admits a strictly
tighter expected one-step growth factor:
\[
1+L_f+\bigl[\alpha\beta_t+(1-\alpha)\mathbb{E}[\beta_K(x_t) | x_t]\bigr]\gamma
<
1+L_f+\beta_t\gamma.
\]
Thus, effective routing \textbf{yields more stable forward dynamics} by replacing a fixed local control norm with a data-dependent routed norm.
\end{proposition}
The key quantity in this result is the routed norm $\beta_K(x_t)$. We enforce $\mathbb{E}[\beta_K(x_t)\vert x_t]<\beta_t$ by using a load-balancing loss \cite{fedus2022_switch, lepikhin2020_gshard}, promoting balanced expert use and better control-expert selection.
\paragraph{Gradient Flow.}
We show that tighter control of expected forward growth bounds the layerwise Jacobian norms in the backpropagated product, yielding more stable gradients as network depth increases.
\begin{corollary}[Improved Gradient Flow]
\label{cor:gradient}
For each time step $t$, denote by
$J_t = \frac{\partial x_{t+1}}{\partial x_t}$ and $J_t^c = \frac{\partial x^c_{t+1}}{\partial x^c_t}$
the Jacobians of the forward dynamics for the MoC and Parallel Control path, respectively. Under \cref{thm:moc_stability}'s assumptions, MoC admits a tighter expected Jacobian norm upper bound than its Parallel Control counterpart.
In particular, $\forall t$, we have $\|J_t^c\| \leq M_t^c \text{ and } \|J_t\| \leq M_t$,
where
$M_t = 1 + L_f\|W_t\| + \alpha\beta_t\gamma + (1-\alpha)\beta_K(x_t)\gamma$ and $M_t^c = 1 + L_f\|W_t\| +\beta_t\gamma$.
Moreover, $ M_t \leq M_t^c, \forall t$, and consequently,
\begin{align}
\prod_{t=1}^{L} \|J_t\| \leq \prod_{t=1}^{L} M_t \leq \prod_{t=1}^{L}M_t^c,
\end{align}
indicating a \textbf{more stable gradient flow} under backpropagation in the MoC path.
\end{corollary}
\paragraph{Representation Error.}
Finally, we link adaptive routing to representation quality: by routing and combining multiple control experts, MoC yields a closer approximation to the ideally adaptive control policy than block-local control alone.
\begin{proposition}[Representation Error Bound]
\label{lem:representation_error}
Consider a network comprising $L$ blocks. Assume $\Delta W^* x$ is the optimal control signal that yields task-optimal representations. Denote selected experts exhibit an average error of
$\bar{\epsilon}_K^2 \coloneqq \mathbb{E}[\|\sum_{k\in S_K(x)} \pi_k(x) \Delta W_k x - \Delta W^*x\|^2 \vert x]$, $\bar{\epsilon}_{c}^2 \coloneqq \mathbb{E}\left[\|\Delta W_tx - \Delta W^*x\|^2 \vert x\right]$ and $\bar{\epsilon}^2 \coloneqq \mathbb{E}\left[\|g_{\text{moc}}(x) - \Delta W^*x\|^2|x\right]$ are the representation errors of Parallel Control and MoC, respectively. Then the representation error of MoC can be bounded as:
\begin{equation}
\bar{\epsilon}^2 \leq \alpha \bar{\epsilon}_{c}^2 + (1-\alpha)\bar{\epsilon}_K^2.
\end{equation}
Furthermore, under \cref{thm:moc_stability}'s assumptions, $\bar{\epsilon}_K^2 < \bar{\epsilon}_{c}^2$, which leads to 
$\bar{\epsilon}^2 < \bar{\epsilon}_{c}^2$.
\end{proposition}

These results show a connected benefit of sparse state-control routing: effective routing tightens forward stability bounds, which improves gradient propagation, while the same routed mixture reduces the gap to the ideal representation. The detailed proofs of \cref{prop4.1:Expressive_moc}, \cref{thm:moc_stability}, \cref{cor:gradient}, and \cref{lem:representation_error} are provided in Appendix (see \cref{Appendix:Proof}).

\begin{algorithm}[t]
\caption{MoC forward pass at block $t$}
\label{alg:moc}
\begin{algorithmic}[1]
\item[] \hspace{-\algorithmicindent}%
\parbox[t]{\dimexpr\linewidth+\algorithmicindent\relax}{%
\hangindent=\InputIndent
\hangafter=1
\textbf{Input:} Hidden state $x_t$, block mapping $f_t$,
$L$ block-aligned control experts $\{g_i\}_{i=1}^{L}$, \\
gating function $\mathbf{g}$ over all $L$ experts, sparsity level $k$,
mixing coefficient $\alpha$.
}
\vspace{0.1em}

\item[] \hspace{-\algorithmicindent}%
\textbf{Output:} Updated hidden state $x_{t+1}$.

\vspace{0.3em}

\State ${\color{clrOwn}p_\mathrm{own}} \gets g_t(x_t)$
\Comment{\textcolor{clrComment}{layer-dedicated expert}}

\State ${\boldsymbol{\pi}}
\gets \mathrm{softmax}\bigl(\mathbf{g}(x_t)\bigr)$
\Comment{\textcolor{clrComment}{dense routing distribution}}

\State ${\mathcal{S}_k}
\gets \operatorname{TopK}\bigl(\boldsymbol{\pi}, k\bigr)$
\Comment{\textcolor{clrComment}{top-$k$ expert selection}}

\State $\pi_i \gets 
\dfrac{\pi_i}
{\sum_{j \in \mathcal{S}_k} \pi_j},
\quad \forall i \in \mathcal{S}_k$
\Comment{\textcolor{clrComment}{re-normalized routing weights}}

\State ${\color{clrRoute}p_\mathrm{route}} 
\gets \sum_{i \in \mathcal{S}_k} \pi_i \cdot g_i(x_t)$
\Comment{\textcolor{clrComment}{sparse expert aggregation}}

\State ${g_\mathrm{moc}}(x_t)
\gets \alpha\,{\color{clrOwn}p_\mathrm{own}} 
+ (1-\alpha)\,{\color{clrRoute}p_\mathrm{route}}$
\Comment{\textcolor{clrComment}{control signal}}

\State ${x_{t+1}} 
\gets x_t + f_t(x_t) + g_\mathrm{moc}(x_t)$
\Comment{\textcolor{clrComment}{update}}

\State \Return ${x_{t+1}}$

\end{algorithmic}
\end{algorithm}

\begin{table*}[!t]
\vspace{-0.1in}
\centering
\caption{Comparison of algorithm performance across commonsense reasoning benchmarks. * marks previously published numbers. Best results are in \textbf{bold}.}
\vspace{0.5em}
\label{tab:main_llm}

\renewcommand\tabcolsep{4.5pt}
\renewcommand\arraystretch{1.3}

\resizebox{\linewidth}{!}{
\begin{tabular}{c|l|c|ccccccccc|cc}
\Xhline{1.2pt}
\rowcolor{cyan!20}
\textbf{Model} & \textbf{Method} & \textbf{\% Params} 
& \textbf{BoolQ} & \textbf{PIQA} & \textbf{SIQA} 
& \textbf{HellaSwag} & \textbf{WinoGrande} 
& \textbf{ARC-e} & \textbf{ARC-c} & \textbf{OBQA} 
& \textbf{Avg ($\uparrow$)}
& \makecell{\textbf{Time ($\downarrow$)}\\ \textbf{(h)}} 
& \makecell{\textbf{GPU Mem. ($\downarrow$)}\\ \textbf{(GB)}} \\
\Xhline{1.2pt}

\multirow{8}{*}{\rotatebox{90}{LLaMA2-7B}}

& LoRA* & 0.83
& 69.8 & 79.9 & 79.5 
& 83.6 & 82.6 & 79.8 
& 64.7 & 81.0 & 77.6 
& 6.1 & 48.0 \\

\cline{2-14}
\noalign{\vskip 0.1mm}

& \cellcolor{gray!10}DoRA* & \cellcolor{gray!10}0.85
& \cellcolor{gray!10}71.8\up{2.0} 
& \cellcolor{gray!10}\textbf{83.7}\up{3.8} 
& \cellcolor{gray!10}76.0\dn{3.5} 
& \cellcolor{gray!10}89.1\up{5.5} 
& \cellcolor{gray!10}82.6\up{0.0} 
& \cellcolor{gray!10}\textbf{83.7}\up{3.9} 
& \cellcolor{gray!10}68.2\up{3.5} 
& \cellcolor{gray!10}\textbf{82.4}\up{1.4} 
& \cellcolor{gray!10}79.7\up{2.1}
& \cellcolor{gray!10}10.5
& \cellcolor{gray!10}80.2 \\

& MoLEx & 0.85
& 69.7\dn{0.1} 
& 81.9\up{2.0} 
& 78.4\dn{1.1} 
& 91.3\up{7.7} 
& 81.5\dn{1.1} 
& 81.5\up{1.7} 
& 65.8\up{1.1} 
& 80.8\dn{0.2} 
& 78.9\up{1.3}
& 9.8 & 77.4 \\

& FacT\_TK & 0.51
& 69.8\up{0.0}
& 80.6\up{0.7}
& 77.8\dn{1.7}
& 90.8\up{7.2}
& 81.8\dn{0.8}
& 79.7\dn{0.1}
& 65.4\up{0.7}
& 78.6\dn{2.4}
& 78.1\up{0.5}
& 6.0 & 48.5 \\

& \cellcolor{gray!10}StructLoRA
& \cellcolor{gray!10}0.83
& \cellcolor{gray!10}70.9\up{1.1}
& \cellcolor{gray!10}82.0\up{2.1}
& \cellcolor{gray!10}78.6\dn{0.9}
& \cellcolor{gray!10}91.6\up{8.0}
& \cellcolor{gray!10}82.1\dn{0.5}
& \cellcolor{gray!10}81.6\up{1.8}
& \cellcolor{gray!10}66.6\up{1.9}
& \cellcolor{gray!10}80.6\dn{0.4}
& \cellcolor{gray!10}79.3\up{1.7}
& \cellcolor{gray!10}8.5
& \cellcolor{gray!10}71.6 \\

& FlexLoRA & 0.83
& 69.5\dn{0.3}
& 81.2\up{1.3}
& 78.1\dn{1.4}
& 92.9\up{9.3}
& 78.7\dn{3.9}
& 82.8\up{3.0}
& 66.5\up{1.8}
& 77.0\dn{4.0}
& 78.3\up{0.7}
& 7.0 & 54.8 \\

& \cellcolor{gray!10}Control* & \cellcolor{gray!10}0.50
& \cellcolor{gray!10}\textbf{72.3}\up{2.5} 
& \cellcolor{gray!10}82.5\up{2.6} 
& \cellcolor{gray!10}79.2\dn{0.3} 
& \cellcolor{gray!10}89.1\up{5.5} 
& \cellcolor{gray!10}83.1\up{0.5} 
& \cellcolor{gray!10}83.0\up{3.2} 
& \cellcolor{gray!10}68.5\up{3.8} 
& \cellcolor{gray!10}79.0\dn{2.0} 
& \cellcolor{gray!10}79.6\up{2.0}
& \cellcolor{gray!10}4.2
& \cellcolor{gray!10}42.6 \\

& \cellcolor{moccolor}\textbf{MoC} & \cellcolor{moccolor}0.51
& \cellcolor{moccolor}71.8\up{2.0} 
& \cellcolor{moccolor}83.2\up{3.3} 
& \cellcolor{moccolor}\textbf{79.6}\up{0.1} 
& \cellcolor{moccolor}\textbf{93.3}\up{9.7} 
& \cellcolor{moccolor}\textbf{84.2}\up{1.6} 
& \cellcolor{moccolor}83.6\up{3.8} 
& \cellcolor{moccolor}\textbf{68.6}\up{3.9} 
& \cellcolor{moccolor}80.4\dn{0.6} 
& \cellcolor{moccolor}\textbf{80.5}\up{2.9}
& \cellcolor{moccolor}5.5
& \cellcolor{moccolor}47.6 \\
\hline

\multirow{8}{*}{\rotatebox{90}{LLaMA3-8B}}

& LoRA* & 0.71
& 70.8 & 85.2 & 79.9 
& 91.7 & 84.3 & 84.2 
& 71.2 & 79.0 & 80.8
& 7.5 & 59.1 \\

\cline{2-14}
\noalign{\vskip 0.1mm}

& \cellcolor{gray!10}DoRA* & \cellcolor{gray!10}0.71
& \cellcolor{gray!10}74.6\up{3.8} 
& \cellcolor{gray!10}89.3\up{4.1} 
& \cellcolor{gray!10}79.9\up{0.0} 
& \cellcolor{gray!10}95.5\up{3.8} 
& \cellcolor{gray!10}85.6\up{1.3} 
& \cellcolor{gray!10}90.5\up{6.3} 
& \cellcolor{gray!10}80.4\up{9.2} 
& \cellcolor{gray!10}85.8\up{6.8} 
& \cellcolor{gray!10}85.2\up{4.4}
& \cellcolor{gray!10}10.5
& \cellcolor{gray!10}69.1 \\

& MoLEx & 0.72
& 70.1\dn{0.7} 
& 86.2\up{1.0} 
& 80.4\up{0.5} 
& 93.2\up{1.5} 
& 84.1\dn{0.2} 
& 84.8\up{0.6} 
& 71.3\up{0.1} 
& 82.2\up{3.2} 
& 81.5\up{0.7}
& 11.0 & 94.9 \\

& FacT\_TK & 0.43
& 70.7\dn{0.1}
& 86.0\up{0.8}
& 79.5\dn{0.4}
& 94.2\up{2.5}
& 84.2\dn{0.1}
& 86.5\up{2.3}
& 73.3\up{2.1}
& 81.8\up{2.8}
& 82.0\up{1.2}
& 7.7 & 57.6 \\

& \cellcolor{gray!10}StructLoRA
& \cellcolor{gray!10}0.71
& \cellcolor{gray!10}70.9\up{0.1}
& \cellcolor{gray!10}86.6\up{1.4}
& \cellcolor{gray!10}79.5\dn{0.4}
& \cellcolor{gray!10}93.9\up{2.2}
& \cellcolor{gray!10}84.0\dn{0.3}
& \cellcolor{gray!10}87.0\up{2.8}
& \cellcolor{gray!10}73.9\up{2.7}
& \cellcolor{gray!10}82.6\up{3.6}
& \cellcolor{gray!10}82.3\up{1.5}
& \cellcolor{gray!10}9.8
& \cellcolor{gray!10}75.4 \\

& FlexLoRA & 0.71
& 73.4\up{2.6}
& 88.5\up{3.3}
& 78.6\dn{1.3}
& 95.3\up{3.6}
& 84.4\up{0.1}
& 90.2\up{6.0}
& 79.2\up{8.0}
& 85.6\up{6.6}
& 84.4\up{3.6}
& 8.6 & 62.4 \\

& \cellcolor{gray!10}Control* & \cellcolor{gray!10}0.42
& \cellcolor{gray!10}74.1\up{3.3} 
& \cellcolor{gray!10}87.8\up{2.6} 
& \cellcolor{gray!10}80.7\up{0.8} 
& \cellcolor{gray!10}95.5\up{3.8} 
& \cellcolor{gray!10}86.0\up{1.7} 
& \cellcolor{gray!10}90.8\up{6.6} 
& \cellcolor{gray!10}80.0\up{8.8} 
& \cellcolor{gray!10}87.8\up{8.8} 
& \cellcolor{gray!10}85.3\up{4.5}
& \cellcolor{gray!10}4.4 
& \cellcolor{gray!10}54.0 \\

& \cellcolor{moccolor}\textbf{MoC} & \cellcolor{moccolor}0.43
& \cellcolor{moccolor}\textbf{75.5}\up{4.7} 
& \cellcolor{moccolor}\textbf{89.5}\up{4.3} 
& \cellcolor{moccolor}\textbf{82.0}\up{2.1} 
& \cellcolor{moccolor}\textbf{96.6}\up{4.9} 
& \cellcolor{moccolor}\textbf{88.8}\up{4.5} 
& \cellcolor{moccolor}\textbf{91.5}\up{7.3} 
& \cellcolor{moccolor}\textbf{81.6}\up{10.4} 
& \cellcolor{moccolor}\textbf{88.0}\up{9.0} 
& \cellcolor{moccolor}\textbf{86.7}\up{5.9}
& \cellcolor{moccolor}6.5
& \cellcolor{moccolor}56.7 \\
\hline

\multirow{8}{*}{\rotatebox{90}{Mistralv3-7B}}

& LoRA & 0.78
& 68.7 & 83.5 & 77.2 & 90.3 & 82.5 & 85.7 & 73.4 & 80.5 & 80.2
& 6.5 & 58.4 \\

\cline{2-14}
\noalign{\vskip 0.1mm}

& \cellcolor{gray!10}DoRA & \cellcolor{gray!10}0.79
& \cellcolor{gray!10}70.3\up{1.5} 
& \cellcolor{gray!10}84.5\up{1.0} 
& \cellcolor{gray!10}77.8\up{0.6} 
& \cellcolor{gray!10}89.9\dn{0.4} 
& \cellcolor{gray!10}83.1\up{0.5} 
& \cellcolor{gray!10}87.1\up{1.4} 
& \cellcolor{gray!10}72.2\dn{1.2} 
& \cellcolor{gray!10}82.4\up{1.9} 
& \cellcolor{gray!10}80.9\up{0.7}
& \cellcolor{gray!10}10.1
& \cellcolor{gray!10}70.4 \\

& MoLEx & 0.80
& 71.4\up{2.7} & 85.4\up{2.0} & 78.2\up{1.0} 
& 90.3\dn{0.0} & 83.5\up{1.0} & 84.1\dn{1.6} 
& 75.0\up{1.5} & 82.7\up{2.2} & 81.3\up{1.1}
& 10.5 & 82.3 \\

& FacT\_TK & 0.48
& 70.5\up{1.8}
& 86.4\up{2.9}
& 80.3\up{3.1}
& 94.1\up{3.8}
& 84.2\up{1.7}
& 83.9\dn{1.8}
& 71.0\dn{2.4}
& 84.4\up{3.9}
& 81.9\up{1.7}
& 6.3 & 56.0 \\

& \cellcolor{gray!10}StructLoRA
& \cellcolor{gray!10}0.78
& \cellcolor{gray!10}70.0\up{1.3}
& \cellcolor{gray!10}83.2\dn{0.3}
& \cellcolor{gray!10}77.4\up{0.2}
& \cellcolor{gray!10}89.9\dn{0.4}
& \cellcolor{gray!10}81.7\dn{0.8}
& \cellcolor{gray!10}87.2\up{1.5}
& \cellcolor{gray!10}74.8\up{1.4}
& \cellcolor{gray!10}81.4\up{0.9}
& \cellcolor{gray!10}80.7\up{0.5}
& \cellcolor{gray!10}9.2
& \cellcolor{gray!10}70.1 \\

& FlexLoRA & 0.78
& 73.9\up{5.2}
& 88.4\up{4.9}
& 79.2\up{2.0}
& \textbf{95.9}\up{5.6}
& 86.0\up{3.5}
& \textbf{90.2}\up{4.5}
& 78.8\up{5.4}
& 86.8\up{6.3}
& 84.9\up{4.7}
& 7.5 & 61.3 \\

& \cellcolor{gray!10}Control & \cellcolor{gray!10}0.46
& \cellcolor{gray!10}70.6\up{1.8} 
& \cellcolor{gray!10}86.9\up{3.5} 
& \cellcolor{gray!10}79.0\up{1.8} 
& \cellcolor{gray!10}94.2\up{3.9} 
& \cellcolor{gray!10}84.8\up{2.2} 
& \cellcolor{gray!10}87.9\up{2.2} 
& \cellcolor{gray!10}76.5\up{3.0} 
& \cellcolor{gray!10}84.8\up{4.3} 
& \cellcolor{gray!10}83.1\up{2.9}
& \cellcolor{gray!10}4.3 
& \cellcolor{gray!10}53.2 \\

& \cellcolor{moccolor}\textbf{MoC} & \cellcolor{moccolor}0.48
& \cellcolor{moccolor}\textbf{74.2}\up{5.5} 
& \cellcolor{moccolor}\textbf{88.9}\up{5.4} 
& \cellcolor{moccolor}\textbf{81.0}\up{3.8} 
& \cellcolor{moccolor}95.6\up{5.3} 
& \cellcolor{moccolor}\textbf{87.9}\up{5.4} 
& \cellcolor{moccolor}88.9\up{3.2} 
& \cellcolor{moccolor}\textbf{79.5}\up{6.1} 
& \cellcolor{moccolor}\textbf{87.8}\up{7.3} 
& \cellcolor{moccolor}\textbf{85.5}\up{5.3}
& \cellcolor{moccolor}5.5
& \cellcolor{moccolor}55.6 \\
\hline

\multirow{8}{*}{\rotatebox{90}{Qwen2.5-14B}}

& LoRA & 0.63
& 73.4 & 90.4 & 82.2 & 96.0 & 88.6 & 94.7 & 86.1 & 90.6 & 87.7
& 9.4 & 62.1 \\

\cline{2-14}
\noalign{\vskip 0.1mm}

& \cellcolor{gray!10}DoRA & \cellcolor{gray!10}0.64
& \cellcolor{gray!10}75.1\up{1.7} 
& \cellcolor{gray!10}90.1\dn{0.3} 
& \cellcolor{gray!10}82.0\dn{0.2} 
& \cellcolor{gray!10}96.3\up{0.3} 
& \cellcolor{gray!10}89.4\up{0.8} 
& \cellcolor{gray!10}95.0\up{0.3} 
& \cellcolor{gray!10}88.2\up{2.1} 
& \cellcolor{gray!10}93.0\up{2.4} 
& \cellcolor{gray!10}88.6\up{0.9}
& \cellcolor{gray!10}14.3
& \cellcolor{gray!10}94.7 \\

& MoLEx & 0.64
& 73.5\up{0.1} 
& 90.5\up{0.1} 
& 81.9\dn{0.3} 
& 95.5\dn{0.5} 
& 88.2\dn{0.4} 
& 94.4\dn{0.3} 
& 85.8\dn{0.3} 
& 92.2\up{1.6} 
& 87.8\up{0.1}
& 20.3 & 124.8 \\

& FacT\_TK & 0.43
& 74.7\up{1.3}
& 90.1\dn{0.3}
& 82.3\up{0.1}
& 96.2\up{0.2}
& 89.6\up{1.0}
& 94.7\up{0.0}
& 88.1\up{2.0}
& 92.4\up{1.8}
& 88.5\up{0.8}
& 8.0 & 61.5 \\

& \cellcolor{gray!10}StructLoRA
& \cellcolor{gray!10}0.63
& \cellcolor{gray!10}74.8\up{1.4}
& \cellcolor{gray!10}89.5\dn{0.9}
& \cellcolor{gray!10}81.0\dn{1.2}
& \cellcolor{gray!10}95.8\dn{0.2}
& \cellcolor{gray!10}89.3\up{0.7}
& \cellcolor{gray!10}94.7\up{0.0}
& \cellcolor{gray!10}87.4\up{1.3}
& \cellcolor{gray!10}92.4\up{1.8}
& \cellcolor{gray!10}88.1\up{0.4}
& \cellcolor{gray!10}13.5
& \cellcolor{gray!10}77.3 \\

& FlexLoRA & 0.63
& 74.2\up{0.8}
& 91.4\up{1.0}
& 82.0\dn{0.2}
& 96.4\up{0.4}
& 87.7\dn{0.9}
& \textbf{96.7}\up{2.0}
& 90.0\up{3.9}
& 95.4\up{4.8}
& 89.2\up{1.5}
& 10.5 & 67.4 \\

& \cellcolor{gray!10}Control & \cellcolor{gray!10}0.43
& \cellcolor{gray!10}72.8\dn{0.6} 
& \cellcolor{gray!10}88.8\dn{1.6} 
& \cellcolor{gray!10}81.3\dn{0.9} 
& \cellcolor{gray!10}94.6\dn{1.4} 
& \cellcolor{gray!10}87.5\dn{1.1} 
& \cellcolor{gray!10}94.4\dn{0.3} 
& \cellcolor{gray!10}86.9\up{0.8} 
& \cellcolor{gray!10}91.2\up{0.6} 
& \cellcolor{gray!10}87.2\dn{0.5}
& \cellcolor{gray!10}6.5
& \cellcolor{gray!10}58.7 \\

& \cellcolor{moccolor}\textbf{MoC} & \cellcolor{moccolor}0.43
& \cellcolor{moccolor}\textbf{75.5}\up{2.1} 
& \cellcolor{moccolor}\textbf{91.8}\up{1.4} 
& \cellcolor{moccolor}\textbf{82.4}\up{0.2} 
& \cellcolor{moccolor}\textbf{96.8}\up{0.8} 
& \cellcolor{moccolor}\textbf{91.0}\up{2.4} 
& \cellcolor{moccolor}96.6\up{1.9} 
& \cellcolor{moccolor}\textbf{90.2}\up{4.1} 
& \cellcolor{moccolor}\textbf{96.4}\up{5.8} 
& \cellcolor{moccolor}\textbf{90.1}\up{2.4}
& \cellcolor{moccolor}8.0
& \cellcolor{moccolor}61.3 \\
\Xhline{1.2pt}
\end{tabular}
}

\vspace{-0.4cm}
\end{table*}

\section{Experiments}
\label{sec5:exp}
This section describes the experimental setup, reports the main results compared to Weight-FT and State-based baselines, and presents ablation studies. Additional results are provided in the Appendix. 

\subsection{Experimental Settings}

\paragraph{Datasets.} We conduct experiments on diverse language and vision benchmarks. For natural language generation (NLG), we use 8 commonsense reasoning datasets: BoolQ \cite{clark2019boolq}, PiQA \cite{bisk2020piqa}, Social IQa \cite{sap-etal-2019-social}, HellaSwag \cite{zellers2019hellaswag}, WinoGrande \cite{sakaguchi2021winogrande}, ARC-E/ARC-C \cite{clark2018think}, and OpenBookQA \cite{mihaylov2018can}. For natural language understanding (NLU), we use GLUE \cite{wang2018glue}, which includes datasets such as SST-2, MRPC,
CoLA, QNLI, RTE, and STS-B. For vision, we use VTAB-1K \cite{zhai2019large}, which contains 19 datasets spanning Natural, Specialized, and Structured domains. These benchmarks provide a comprehensive evaluation across language and vision tasks.
\paragraph{Backbone models.} For NLG, we use LLaMA2‑7B \cite{touvron2023llama}, LLaMA3‑8B \cite{grattafiori2024llama}, Mistral‑7B Instruct \cite{jiang2023mistral7b}, and Qwen2.5‑14B Instruct \cite{qwen2025qwen25technicalreport}. For NLU, we evaluate RoBERTa and DeBERTa in both base and large variants. For vision tasks, we adopt ViT‑Tiny and ViT‑Base as standard backbones.

\paragraph{Baselines.}
Across domains, we compare against some recent, representative adaptation methods, including LoRA \cite{hu2021loralowrankadaptationlarge}, DoRA \cite{liu2024doraweightdecomposedlowrankadaptation}, MoLEx \cite{teo2025molex},  FacT\_TK \citep{jie2023fact}, StructLoRA \citep{xiao2026not}, FlexLoRA \citep{liu2026flexlora}, and the recent Parallel Control \cite{zhangweight} (abbreviated as Control).

\paragraph{Evaluation Metrics and Analysis.}
For NLG and vision tasks, we evaluate performance using accuracy for eight commonsense reasoning datasets and 19 visual datasets, respectively. For NLU tasks, we follow standard evaluation protocols and report task‑specific metrics: Matthew’s correlation coefficient for CoLA, Pearson correlation for STS‑B, and accuracy for all remaining tasks. 

\subsection{Main Results}
We use LoRA as the common reference baseline; superscripts $\uparrow$ and $\downarrow$ denote improvement and degradation relative to LoRA, respectively.

\paragraph{NLG Tasks.} 
We evaluate MoC on NLG tasks (\cref{tab:main_llm}). MoC consistently outperforms prior methods across backbones from LLaMA2-7B to Qwen2.5-14B, beating LoRA by $2.4\%$--$6.0\%$ and the strongest baseline by $0.6\%$--$1.4\%$ with comparable compute. While MoLEx also enables cross-block communication, it costs up to $>2\times$ time and memory on Qwen2.5-14B; MoC achieves better performance with superior scalability. 
Beyond accuracy, \cref{fig:inference_throughput} compares the inference throughput of MoC with that of representative baselines offering competitive efficiency--performance trade-offs across model scales and batch sizes.
While state-based methods remain less efficient than fully merged Weight-FT, MoC consistently outperforms MoLEx in throughput--even after merging--improving the accuracy–efficiency trade-off and narrowing the gap to mergeable fine-tuning. Besides, a detailed efficiency comparison between MoC and representative baselines is provided in \Cref{app:additional_experiments}.
\paragraph{NLU Tasks.} 
Next, we evaluate the MoC against previous approaches on the GLUE benchmark in \cref{tab:glue_roberta}. MoC consistently outperforms competing baselines across all six datasets, achieving average gains of 1.5\% and 2.0\% over LoRA for RoBERTa-Base and RoBERTa-Large, respectively, and outperforming the strongest competing baseline by 1.2\%. Together, these results highlight the effectiveness of MoC across both NLG and NLU settings. We further report DeBERTa results in Appendix (see \cref{tab:glue_deberta} for details).

\begin{figure*}[t]
\vspace{-0.45cm}
\centering    \includegraphics[width=0.9\linewidth]{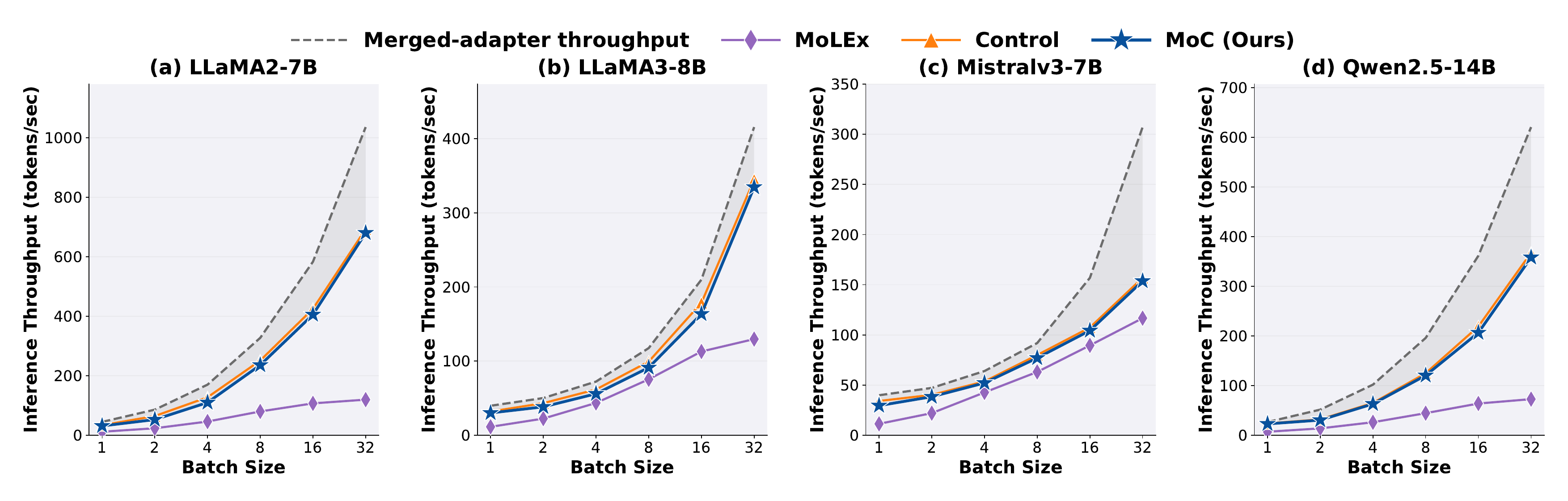}
    \caption{Inference Throughput of MoC compared to prior methods across LLMs and batch size.}
    \label{fig:inference_throughput}
\vspace{-0.47cm}        
\end{figure*}

\begin{table*}[t]
\centering
\caption{GLUE results with RoBERTa-Base/Large backbones. Values are medians of five runs (different seeds); * marks previously published numbers. Best results are in \textbf{bold}. }
\vspace{0.5em}
\label{tab:glue_roberta}

\renewcommand\tabcolsep{10pt}
\renewcommand\arraystretch{1.2}

\resizebox{0.8\textwidth}{!}{
\begin{tabular}{c|l|ccccccc}
\Xhline{1.2pt}\rowcolor{cyan!20}
\textbf{Model} & \textbf{Method} & \textbf{SST-2 }& \textbf{MRPC} & \textbf{CoLA} & \textbf{QNLI} & \textbf{RTE} & \textbf{STS-B} & \textbf{Avg} ($\uparrow$) \\
\Xhline{1.2pt}

\multirow{5}{*}{\rotatebox{0}{RoBERTa-Base}}

& LoRA* 
& 95.1
& 89.7
& 63.4
& 93.3
& 78.4
& 91.5
& 85.2 \\

\cline{2-9}
\noalign{\vskip 0.1mm}

& \cellcolor{gray!10}DoRA* 
& \cellcolor{gray!10}95.0\dn{0.1}
& \cellcolor{gray!10}89.7\up{0.0}
& \cellcolor{gray!10}64.9\up{1.5}
& \cellcolor{gray!10}92.9\dn{0.4}
& \cellcolor{gray!10}79.2\up{0.8}
& \cellcolor{gray!10}91.3\dn{0.2}
& \cellcolor{gray!10}85.5\up{0.3} \\

& MoLEx* 
& 95.4\up{0.3}
& 89.8\up{0.1}
& 64.8\up{1.4}
& 93.2\dn{0.1}
& 77.3\dn{1.1}
& 91.0\dn{0.5}
& 85.3\up{0.1} \\

& \cellcolor{gray!10}Control* 
& \cellcolor{gray!10}95.3\up{0.2}
& \cellcolor{gray!10}89.3\dn{0.4}
& \cellcolor{gray!10}65.3\up{1.9}
& \cellcolor{gray!10}93.1\dn{0.2}
& \cellcolor{gray!10}76.9\dn{1.5}
& \cellcolor{gray!10}90.6\dn{0.9}
& \cellcolor{gray!10}85.1\dn{0.1} \\

& \cellcolor{moccolor}\textbf{MoC} 
& \cellcolor{moccolor}\textbf{95.6}\up{0.5}
& \cellcolor{moccolor}\textbf{90.7}\up{1.0}
& \cellcolor{moccolor}\textbf{66.7}\up{3.3}
& \cellcolor{moccolor}\textbf{93.4}\up{0.1}
& \cellcolor{moccolor}\textbf{81.9}\up{3.5}
& \cellcolor{moccolor}\textbf{91.5}\up{0.0}
& \cellcolor{moccolor}\textbf{86.7}\up{1.5} \\

\hline

\multirow{5}{*}{\rotatebox{0}{RoBERTa-Large}}

& LoRA* 
& 96.2
& 90.2
& 68.2
& 94.8
& 85.2
& 92.3
& 87.8 \\

\cline{2-9}
\noalign{\vskip 0.1mm}

& \cellcolor{gray!10}DoRA* 
& \cellcolor{gray!10}96.6\up{0.4}
& \cellcolor{gray!10}90.9\up{0.7}
& \cellcolor{gray!10}68.8\up{0.6}
& \cellcolor{gray!10}95.2\up{0.4}
& \cellcolor{gray!10}84.7\dn{0.5}
& \cellcolor{gray!10}92.5\up{0.2}
& \cellcolor{gray!10}88.1\up{0.3} \\

& MoLEx* 
& 96.4\up{0.2}
& 91.4\up{1.2}
& 68.2\up{0.0}
& 94.8\up{0.0}
& 87.1\up{1.9}
& 92.0\dn{0.3}
& 88.3\up{0.5} \\

& \cellcolor{gray!10}Control* 
& \cellcolor{gray!10}96.7\up{0.5}
& \cellcolor{gray!10}90.9\up{0.7}
& \cellcolor{gray!10}69.9\up{1.7}
& \cellcolor{gray!10}95.2\up{0.4}
& \cellcolor{gray!10}86.3\up{1.1}
& \cellcolor{gray!10}92.5\up{0.2}
& \cellcolor{gray!10}88.6\up{0.8} \\

& \cellcolor{moccolor}\textbf{MoC} 
& \cellcolor{moccolor}\textbf{97.2}\up{1.0}
& \cellcolor{moccolor}\textbf{92.6}\up{2.4}
& \cellcolor{moccolor}\textbf{71.4}\up{3.2}
& \cellcolor{moccolor}\textbf{95.2}\up{0.4}
& \cellcolor{moccolor}\textbf{89.9}\up{4.7}
& \cellcolor{moccolor}\textbf{92.5}\up{0.2}
& \cellcolor{moccolor}\textbf{89.8}\up{2.0} \\

\Xhline{1.2pt}
\end{tabular}
}
\vspace{-0.25in}
\end{table*}
\paragraph{Vision Tasks.}
To test whether MoC’s gains extend beyond text, we evaluate it on VTAB-1K with ViT-Tiny and ViT-Base. As shown in \cref{tab:vtab_vit}, MoC achieves the best average accuracy across 19 datasets spanning Natural, Specialized, and Structured domains, improving over LoRA by $1.3\%$ and $1.4\%$, and over the strongest baseline by $1.3\%$ and $1.2\%$ on ViT-Tiny and ViT-Base, respectively. MoC also shows competitive training time and memory usage compared with LoRA, DoRA, and MoLEx, highlighting its advantages in performance, time, and memory efficiency. Detailed per-dataset results and plots are provided in the Appendix; see \cref{fig:vit-radar-combined,table:vit_detailed_accuracy}.

Overall, these results show that MoC provides a unified and efficient adaptation strategy across language and vision tasks.
\begin{figure}[t]
    \centering
    \includegraphics[width=\linewidth]{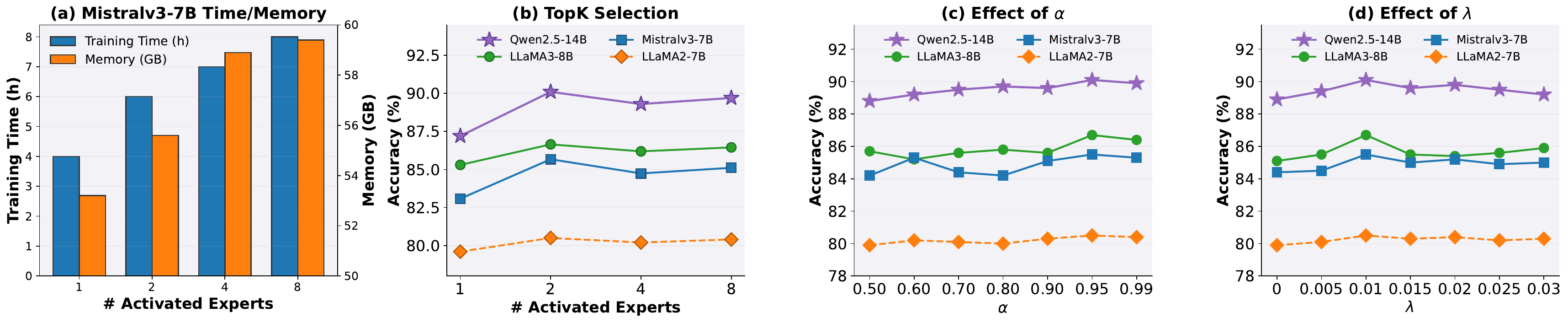}
\caption{
Ablation analysis of MoC under different configurations. 
(a) Training-time and GPU-memory overhead as the number of activated experts varies, 
reported on Mistralv3-7B. 
(b--d) Performance analysis of MoC across different LLM backbones, including 
(b) TopK Selection, 
(c) the effect of the coefficient $\alpha$, 
and (d) the effect of the coefficient $\lambda$.
}    \label{fig:ablation_topk_alpha_lambda}
\end{figure}

\begin{table}[t]
\vspace{-0.1in}
\centering
\caption{Comparison on VTAB-style benchmarks.}
\vspace{0.5em}
\label{tab:vtab_vit}
\resizebox{0.8\linewidth}{!}{
\begin{tabular}{c|l|cccc|cc}
\Xhline{1.2pt}
\rowcolor{cyan!20}
\textbf{Model} & \textbf{Method} 
& \textbf{Natural} ($\uparrow$) 
& \textbf{Specialized} ($\uparrow$) 
& \textbf{Structured} ($\uparrow$) 
& \textbf{Avg} ($\uparrow$) 
& \textbf{Time (m)} ($\downarrow$)
& \textbf{GPU Mem. (GB)} ($\downarrow$)\\
\Xhline{1.2pt}

\multirow{5}{*}{ViT-Tiny}

& LoRA 
& 67.5 & 84.6 & 49.9 & 63.7 & 1.5 & 5.4 \\

\cline{2-8}
\noalign{\vskip 0.1mm}

& \cellcolor{gray!10}DoRA 
& \cellcolor{gray!10}67.0\dn{0.5} 
& \cellcolor{gray!10}84.2\dn{0.4} 
& \cellcolor{gray!10}49.9\up{0.0} 
& \cellcolor{gray!10}63.4\dn{0.3} 
& \cellcolor{gray!10}1.9\dn{0.4} 
& \cellcolor{gray!10}9.2\dn{3.8} \\

& MoLEx 
& 68.1\up{0.6} 
& 84.4\dn{0.2} 
& 49.4\dn{0.5} 
& 63.6\dn{0.1} 
& 2.6\dn{1.1} 
& 11.6\dn{6.2} \\

& \cellcolor{gray!10}Control 
& \cellcolor{gray!10}65.7\dn{1.8} 
& \cellcolor{gray!10}84.4\dn{0.2} 
& \cellcolor{gray!10}47.1\dn{2.8} 
& \cellcolor{gray!10}61.8\dn{1.9} 
& \cellcolor{gray!10}1.0\up{0.5} 
& \cellcolor{gray!10}3.3\up{2.1} \\

& \cellcolor{moccolor}\textbf{MoC} 
& \cellcolor{moccolor}\textbf{70.0}\up{2.5} 
& \cellcolor{moccolor}\textbf{85.9}\up{1.3} 
& \cellcolor{moccolor}\textbf{50.2}\up{0.3} 
& \cellcolor{moccolor}\textbf{65.0}\up{1.3} 
& \cellcolor{moccolor}1.4\up{0.1} 
& \cellcolor{moccolor}3.6\up{1.8} \\

\hline

\multirow{5}{*}{ViT-Base}

& LoRA 
& 77.9 & 85.9 & 48.1 & 67.0 & 4.5 & 21.6 \\

\cline{2-8}
\noalign{\vskip 0.1mm}

& \cellcolor{gray!10}DoRA 
& \cellcolor{gray!10}78.1\up{0.2} 
& \cellcolor{gray!10}85.1\dn{0.8} 
& \cellcolor{gray!10}47.1\dn{0.9} 
& \cellcolor{gray!10}66.6\dn{0.4} 
& \cellcolor{gray!10}8.7\dn{4.2} 
& \cellcolor{gray!10}36.6\dn{15.0} \\

& MoLEx 
& 78.1\up{0.1} 
& 86.9\up{1.0} 
& 47.5\dn{0.6} 
& 67.1\up{0.1} 
& 10.6\dn{6.1} 
& 39.9\dn{18.3} \\

& \cellcolor{gray!10}Control 
& \cellcolor{gray!10}77.6\dn{0.3} 
& \cellcolor{gray!10}86.5\up{0.6} 
& \cellcolor{gray!10}48.5\up{0.4} 
& \cellcolor{gray!10}67.2\up{0.2} 
& \cellcolor{gray!10}3.9\up{0.6} 
& \cellcolor{gray!10}13.8\up{7.8} \\

& \cellcolor{moccolor}\textbf{MoC} 
& \cellcolor{moccolor}\textbf{79.0}\up{1.1} 
& \cellcolor{moccolor}\textbf{87.6}\up{1.8} 
& \cellcolor{moccolor}\textbf{49.4}\up{1.4} 
& \cellcolor{moccolor}\textbf{68.4}\up{1.4} 
& \cellcolor{moccolor}4.2\up{0.3} 
& \cellcolor{moccolor}13.8\up{7.8} \\

\Xhline{1.2pt}
\end{tabular}
}
\vspace{-0.5cm}
\end{table}

\subsection{Ablation Studies}
\label{main:ablation}
\paragraph{TopK Selection.}
We analyze Top-$K$ selection across LLM backbones and report training time and GPU memory on Mistralv3-7B to evaluate scalability. As shown in \cref{fig:ablation_topk_alpha_lambda}(b), MoC performs best with two activated experts, followed by eight. In MoC, Top-$K=1$ pairs the block-local controller with one routed controller from other blocks, similar to Top-2 routing in sparse MoE and offering a strong accuracy–efficiency trade-off \cite{lepikhin2020_gshard, teo2024_momentumsmoe, zhou2023brainformers}. Unlike standard MoE, MoC selects experts across blocks rather than within a layer, so increasing Top-$K$ changes both control capacity and inter-block information exchange. \\
We also test $K\in\{0,1,3,7\}$ (correspond to the number of activated experts $\in\{1,2,4,8\}$), where $K=0$ uses only the local controller (\textit{Control baseline}); see \cref{fig:ablation_topk_alpha_lambda}(a). Training time and memory increase moderately with Top-$K$ and remain well below MoLEx in \cref{tab:main_llm}, indicating efficient scaling. Routing visualizations further show adaptive use of local and global information, supporting cross-block control (see \cref{fig:crosstype_llama3_top1,fig:crosstype_llama3_top3,fig:crosstype_llama3_top7,fig:grid_attn_llama3_top1_all_datasets,fig:grid_ffn_llama3_top1_all_datasets,fig:grid_attn_llama3_top3_all_datasets,fig:grid_ffn_llama3_top3_all_datasets,fig:grid_attn_llama3_top7_all_datasets,fig:grid_ffn_llama3_top7_all_datasets} in Appendix).

\paragraph{Mixing Coefficient $\alpha$.}
We study the sensitivity of MoC to the mixing coefficient $\alpha$. Motivated by \cref{prop4.1:Expressive_moc}, we vary $\alpha \in \{0.5, 0.6, 0.7, 0.8, 0.9, 0.95, 0.99\}$ while fixing $\lambda=0.01$. As shown in \cref{fig:ablation_topk_alpha_lambda} (c), MoC achieves the best performance at $\alpha=0.95$ across different LLMs. It also remains competitive or superior to the baselines in \cref{tab:main_llm} over a wide range of $\alpha$ values, indicating robustness to $\alpha$ while benefiting from the expressive inter-block communication characterized in \cref{prop4.1:Expressive_moc}.

\paragraph{Load-Balancing Coefficient $\lambda$.}
We examine the MoC's sensitivity to the load-balancing coefficient $\lambda$ by fixing $\alpha=0.95$ and varying $\lambda \in \{0, 0.005, 0.01, \ldots, 0.03\}$, where $\lambda=0$ removes load balancing. As shown in \cref{fig:ablation_topk_alpha_lambda} (d), load balancing generally improves performance, indicating better expert utilization and routing and closer alignment with the conditions in \cref{thm:moc_stability}, \cref{cor:gradient}, and \cref{lem:representation_error}. Performance peaks at $\lambda=0.01$, suggesting a small coefficient suffices, while larger values can hinder the main objective.

\paragraph{MoC Variants.}
We compare MoC variants that control attention, FFN, or both (\cref{tab:moc_variants}); FFN control consistently beats attention-only, and MoC (FFN+ATTN) with a shared gate performs best, matching the view that attention routes information while FFNs drive nonlinear semantic transformation \cite{geva2021transformer, lu2019understanding}. Shared gating over both modules enables richer cross-layer control and finer expert selection, strengthening cross-layer communication in line with fine-grained SMoE principles (\cref{sec4.2:moc}).
\begin{table}[th]
\vspace{-0.2cm}
\centering

\begin{minipage}{0.47\linewidth}
\centering
\caption{Comparison of different MoC variants on LLaMA3-8B and Mistralv3-7B.}
\label{tab:moc_variants}
\resizebox{\linewidth}{!}{
\begin{tabular}{l|cc}
\Xhline{1.2pt}
\rowcolor{cyan!20}
\textbf{MoC Variant} 
& \textbf{LLaMA3-8B} ($\uparrow$)
& \textbf{Mistralv3-7B} ($\uparrow$) \\
\Xhline{1.2pt}

ATTN 
& 85.9 & 84.9 \\

\rowcolor{gray!10}
FFN 
& 86.2 & 85.1 \\

ATTN+FFN w/o sharing 
& 86.4 & 84.9 \\

\rowcolor{moccolor}
\textbf{ATTN+FFN} 
& \textbf{86.7} 
& \textbf{85.5} \\

\Xhline{1.2pt}
\end{tabular}
}
\end{minipage}
\hfill
\begin{minipage}{0.47\linewidth}
\centering
\caption{Different control activations on VTAB benchmarks using ViT-Base backbone.}
\label{tab:ablation_control_activation}
\resizebox{\linewidth}{!}{
\begin{tabular}{l|cccc}
\Xhline{1.2pt}
\rowcolor{cyan!20}
\textbf{Activation} 
& \textbf{Natural} ($\uparrow$)
& \textbf{Specialized} ($\uparrow$)
& \textbf{Structured} ($\uparrow$)
& \textbf{Avg} ($\uparrow$) \\
\Xhline{1.2pt}

\rowcolor{moccolor}
\textbf{Identity} 
& 79.0 
& \textbf{87.6} 
& \textbf{49.4} 
& \textbf{68.4} \\

\rowcolor{gray!10}
ReLU 
& \textbf{79.1} 
& 86.5 
& 48.9 
& 67.9 \\

GeLU 
& 78.7 
& 86.8 
& \textbf{49.4 }
& 68.1 \\

\rowcolor{gray!10}
SiLU 
& \textbf{79.1}
& 86.6
& 48.8
& 67.9 \\

\Xhline{1.2pt}
\end{tabular}
}
\end{minipage}
\vspace{-0.65cm}
\end{table}

\begin{table}[th]
\centering

\begin{minipage}{0.45\linewidth}
\centering
\caption{Ablation of different cross-layer communication designs on Mistralv3-7B-Instruct.}
\label{tab:ablation_routing}
\resizebox{\linewidth}{!}{
\begin{tabular}{l|cc}
\Xhline{1.2pt}
\rowcolor{cyan!20}
\textbf{Method}
& \textbf{Avg. Acc.} ($\uparrow$)
& \textbf{GPU Mem.} ($\downarrow$) \\
\Xhline{1.2pt}

Control
& 83.1 & 53.2 \\

\rowcolor{gray!10}
Uniform Control
& 84.3 & 53.8 \\

Fixed Control
& 84.2 & 55.0 \\

\rowcolor{gray!10}
MoC w/o lb
& 84.4 & 55.4 \\

\rowcolor{moccolor}
\textbf{MoC}
& \textbf{85.5}
& 55.6 \\

\Xhline{1.2pt}
\end{tabular}
}
\end{minipage}
\hfill
\begin{minipage}{0.52\linewidth}
\centering
\caption{Effect of expert-pool size on Mistralv3-7B-Instruct. More groups yield smaller routing pools.}
\label{tab:ablation_expert_pool}
\renewcommand\tabcolsep{4.5pt}
\renewcommand\arraystretch{1.35}
\resizebox{\linewidth}{!}{
\begin{tabular}{cc|ccc}
\Xhline{1.2pt}
\rowcolor{cyan!20}
\textbf{\# Groups}
& \textbf{\# Experts}
& \textbf{Avg. Acc.} ($\uparrow$)
& \textbf{Time} ($\downarrow$)
& \textbf{GPU Mem.} ($\downarrow$) \\
\Xhline{1.2pt}

\rowcolor{moccolor}
\textbf{1}
& \textbf{64}
& \textbf{85.5}
& 6.1
& 55.6 \\

\rowcolor{gray!10}
2
& 32
& 85.1
& 6.1
& 55.6 \\

4
& 16
& 85.0
& 6.1
& 55.4 \\

\rowcolor{gray!10}
8
& 8
& 84.8
& 6
& 55.4 \\

16
& 4
& 84.8
& 6
& 55.4 \\

\Xhline{1.2pt}
\end{tabular}
}
\end{minipage}
\end{table}

\paragraph{Routing Design and Expert-Pool Size.}
We investigate the individual contributions of cross-layer communication, adaptive routing, and expert-pool size on Mistralv3-7B-Instruct. As shown in \Cref{tab:ablation_routing}, all communication variants outperform the local Control baseline. Uniform and Fixed Control achieve accuracies of $84.3\%$ and $84.2\%$, showing that cross-layer alignment alone is beneficial, while adaptive routing without load balancing reaches $84.4\%$. Adding load balancing further improves MoC to 85.5\% (+1.1 points), supporting balanced input-dependent expert selection and the motivation of \Cref{thm:moc_stability} and \Cref{cor:gradient}, with only modest memory overhead.
\Cref{tab:ablation_expert_pool} further shows that reducing the expert pool from 64 to 4 lowers accuracy from $85.5\%$ to $84.8\%$, while training time and memory remain nearly unchanged. This suggests that MoC benefits from broad access to cross-layer controls rather than additional computation, highlighting the value of flexible global routing.

\paragraph{Control Types.}
We study MoC with alternative controller designs. In our main setup, each controller is a linear low-rank projection, offering a strong performance–efficiency trade-off (\cref{sec:4.1}). More generally, we define
$g_i(x)=\sigma_2(B_i\sigma_1(A_i x))$,
where $\sigma_1$ and $\sigma_2$ can be linear or nonlinear. For simplicity, we set $\sigma_2$ to the identity and consider $\sigma_1 \in \{\text{Identity}, \text{ReLU}, \text{GeLU}, \text{SiLU}\}$, yielding a generalized family of nonlinear adapters. As shown in \cref{tab:ablation_control_activation}, MoC extends naturally to these variants, and nonlinear controllers are competitive across tasks. Nonetheless, the linear controller remains highly effective while preserving simplicity and efficiency. We leave theoretical analysis of nonlinear controllers to future work.

\section{Conclusion}
In this paper, we identify two limitations of efficient fine-tuning: (i) state-based methods restrict adaptation to per-block control, limiting multi-block coordination; and (ii) existing cross-block communication methods often incur substantial compute and memory overhead. We introduce Mixture-of-Control (MoC), a lightweight framework that treats block-wise control states as experts and adaptively combines local and global signals through sparse cross-block routing. MoC enables efficient multi-block coordination while preserving the memory benefits of state-based fine-tuning. Experiments across language and vision domains with diverse backbones show that MoC consistently outperforms prior baselines with minimal overhead and favorable inference throughput. Taken together, these results highlight control-driven cross-block communication as a promising direction for scalable Transformer adaptation. Future work includes exploring more expressive control architectures and further narrowing the inference-efficiency gap to mergeable fine-tuning methods.
\vspace{-0.5cm}
\section*{Limitations.}
MoC is subject to several limitations. First, unlike mergeable weight-based methods such as LoRA, its state-dependent control and routing modules remain active at inference, introducing overhead that may reduce throughput, especially with many control sites. Developing more efficient routing and control mechanisms is therefore important for practical deployment. Second, our theory currently focuses on linear low-rank controllers. Although nonlinear variants can further improve performance (see \Cref{tab:ablation_control_activation}), establishing comparable expressivity, approximation, and stability guarantees remains open. Third, the mixing coefficient $\alpha$ and load-balancing coefficient $\lambda$ are currently selected manually. While these choices are guided by our theory and standard MoE practice and remain robust across the tested ranges, dynamically or jointly learning them could improve adaptation across models and tasks. Finally, extending cross-layer communication to architectures such as looped Transformers \citep{xu2024expressive,jeddi2026loopformer} raises broader research questions about representation compatibility, depth-wise learning dynamics, and input-dependent communication.

\begin{ack}
\textbf{Funding and Support.}
This research was conducted in part at Qualcomm AI
Research, an initiative of Qualcomm Technologies,
Inc., with research resources provided by the company.

\textbf{Competing Interests.}
All authors were affiliated with Qualcomm Technologies,
Inc. during the conduct of this research.
The authors declare no other competing interests.
\end{ack}




\bibliographystyle{unsrtnat}
\bibliography{reference}


\appendix
\clearpage
\begin{center}
    \rule{\textwidth}{1pt}\\[0.6em]
    {\large \textbf{Mixture-of-Control: State-Aware Fine-Tuning for Transformer-based Models}}\\[0.4em]
    {\large Supplementary Material}\\[0.6em]
    \rule{\textwidth}{1pt}
\end{center}

\section{Proofs for Theoretical Results}
\label{Appendix:Proof}
\subsection{\cref{prop4.1:Expressive_moc}}
\begin{proof}
For simplicity, in the linear case, we take $f_t$ to be an identity-like mapping. Thus, MoLEx update is  simplified into:
\begin{align}
z_{t+1} &= z_t + [\alpha (W_t + \Delta W_t') + (1 - \alpha) (W_{h_t} + \Delta W_{h_t}')]z_t \\
&= z_t + W_t z_t + [\alpha (\Delta W_t') + (1 - \alpha) (W_{h_t} - W_t + \Delta W_{h_t}')]z_t 
\end{align}

Denote $M_t = \alpha (\Delta W_t') + (1 - \alpha) (W_{h_t} - W_t + \Delta W_{h_t}')$, for $t = \overline{1, L}$. Similarly, the generalized state-based update is formulated as follows:
\begin{align*}
    x_{t+1} &= x_t + W_t x_t + \alpha \Delta W_t x_t + (1 - \alpha) \Delta W_{h_t} x_t \\ 
    &= x_t + W_t x_t +(\alpha \Delta W_t + (1 - \alpha) \Delta W_{h_t})x_t
\end{align*}

Consequently, under the initial condition $x_1=z_1$, the equivalence between the two systems can be demonstrated by the following set of equations:
$$
\begin{cases}
    \alpha \Delta W_1 + (1-\alpha)\Delta W_{h_1} = M_1\\
    \alpha \Delta W_2 + (1-\alpha)\Delta W_{h_2} = M_2\\ 
    ... \\
    \alpha \Delta W_L + (1-\alpha)\Delta W_{h_L} = M_L. 
\end{cases}
$$
This admits the following matrix representation: 
$$
\begin{bmatrix}
    \alpha & 0 & 0 & \ldots & 1-\alpha & 0 & \ldots & 0 \\ 
    0 & \alpha & 0 & \ldots & 0 & 1 - \alpha  & \ldots & 0 \\ 
    0 & 0 & \alpha & \ldots & 0 & 0 & \ldots & \ldots \\ 
    \ldots & \ldots & \ldots & \ldots & \ldots & \ldots & \ldots & \ldots \\ 
    \ldots & \ldots & \ldots & \ldots & \ldots & \ldots & \ldots & \ldots \\ 
    \ldots & \ldots & \ldots & \ldots & \ldots & \ldots & \ldots & \ldots \\ 
    \vdots & \vdots & \vdots & \vdots & \vdots & \vdots & \vdots & \vdots \\ 
    \ldots & \ldots & \ldots &  1 - \alpha & 0 & 0 & \ldots & \alpha
\end{bmatrix} \begin{bmatrix}
    \Delta W_1 \\
    \Delta W_2 \\ 
    \Delta W_3 \\
    \Delta W_4 \\
    \Delta W_5 \\
    \Delta W_6 \\
    \vdots \\
    \Delta W_L
\end{bmatrix} = \begin{bmatrix}
     M_1 \\
     M_2 \\ 
     M_3 \\
     M_4 \\
     M_5 \\
     M_6 \\
    \vdots \\
     M_L
\end{bmatrix}.
$$
Denote $$A = 
\begin{bmatrix}
    \alpha & 0 & 0 & \ldots & 1-\alpha & 0 & \ldots & 0 \\ 
    0 & \alpha & 0 & \ldots & 0 & 1 - \alpha  & \ldots & 0 \\ 
    0 & 0 & \alpha & \ldots & 0 & 0 & \ldots & \ldots \\ 
    \ldots & \ldots & \ldots & \ldots & \ldots & \ldots & \ldots & \ldots \\ 
    \ldots & \ldots & \ldots & \ldots & \ldots & \ldots & \ldots & \ldots \\ 
    \ldots & \ldots & \ldots & \ldots & \ldots & \ldots & \ldots & \ldots \\ 
    \vdots & \vdots & \vdots & \vdots & \vdots & \vdots & \vdots & \vdots \\ 
    \ldots & \ldots & \ldots &  1 - \alpha & 0 & 0 & \ldots & \alpha
\end{bmatrix},$$
where $A = \alpha I + (1 - \alpha) B$, where $I \in \mathbb{R}^{Ld \times Ld}$ is the identity matrix, and $B \in \mathbb{R}^{Ld \times Ld}$ is a matrix with exactly $Ld$ entries equal to $1$ (with positions determined by $W_{h_t}$), and all other entries equal to $0$. 

Next, we prove that $\alpha > \tfrac{1}{2}$ ensures $\det(A) \neq 0$. Suppose, for contradiction, that there exists a nonzero vector
$$
\mathbf{x} = \begin{bmatrix} x_1 \cr x_2 \cr \ldots \cr x_{{Ld}} \end{bmatrix} \neq \mathbf{0},
$$
such that
$$
A \mathbf{x} = 0.
$$
Then,
$$
(\alpha I + (1 - \alpha) B)\mathbf{x} = 0
\quad \Longleftrightarrow \quad
B \mathbf{x} = -\frac{\alpha}{1 - \alpha} \mathbf{x}.
$$
Writing this component-wise, we obtain
$$\begin{bmatrix}x_{i_1}\cr x_{i_2}\cr \vdots \cr x_{i_{Ld}}\end{bmatrix} = \dfrac{\alpha}{1 - \alpha} \begin{bmatrix}-x_1 \cr -x_2 \cr \vdots \cr -x_{Ld}\end{bmatrix},$$
where $\{i_1, i_2, \ldots, i_{Ld}\} \subseteq \{1, 2, \ldots, Ld\}$.

Taking the maximum absolute value on both sides yields
$$\max_{1 \le j \le Ld} |x_{i_j}| = \frac{\alpha}{1 - \alpha} \max_{1 \le j \le Ld} |x_j|.$$
Since $\alpha > \tfrac{1}{2}$ implies $\frac{\alpha}{1 - \alpha} > 1$, we obtain
$$\max_{1 \le j \le Ld} |x_{i_j}| > \max_{1 \le j \le Ld} |x_j|,$$
which is a contradiction. Therefore, no such nonzero $\mathbf{x}$ exists, and hence $\det(A) \neq 0$. Consequently, there exists a \emph{unique} solution \(\{\Delta W_k\}_{k=1}^L\) that establishes representational equivalence between the generalized state-based formulation and MoLEx.

In the nonlinear case, we show that the equivalence holds approximately when the adapters \(\{\Delta W_k'\}_{k=1}^L\) are low-norm perturbations, using local linearization around the frozen pretrained path. Assume that each activation function \(f_k\) (\(k=1,\dots,L\)) is twice continuously differentiable, with
\begin{itemize}
    \item Lipschitz constant \(L_f < \infty\) (uniform over \(k\)): \(\|f_k(a) - f_k(b)\| \leq L_f \|a - b\|\) for all \(a,b\),
    \item Bounded Hessian norm \(H < \infty\) (uniform over \(k\)): \(\|\nabla^2 f_k(z)\| \leq H\) for all \(z\) (operator norm).
    \item Assume the MoLEx adapters are small: $\|\Delta W_k'\| \leq \xi$ with $\xi > 0$, $\forall \, k = \overline{1, L}$. Assume the matrix $A$ above is invertible. Assume bounded intermediate states $\mathbb{E}[\|x_t\|^2] \leq B < \infty$ and \emph{pretrained mixing closeness}:
    \[
    \mathbb{E}\left[ \big\|f_{h_t}(W_{h_t} x_t) - f_t(W_t x_t)\big\|^2 \right] \leq \gamma^2 \mathbb{E}[\|x_t\|^2] + \beta^2
    \]
    for small $\gamma, \beta \geq 0$ independent of $\xi$.
\end{itemize}
These conditions are satisfied by common activation functions (ReLU, GeLU, SwiGLU, etc.) on bounded domains or under mild regularization. Furthermore, the final assumption is well‑motivated: recent work \cite{he2024matters} shows that the outputs of Transformer blocks exhibit high similarity, which has driven further investigation into model pruning. Motivated by this observation, we impose a bound on the block outputs, referred to as the $``$\textit{pretrained mixing closeness}$"$ assumption.

We construct the state-based adapters \(\{\Delta W_k\}_{k=1}^L\) exactly as in the linear case: solve the invertible linear system derived above (possible since \(\alpha > \tfrac{1}{2}\)) to match the \emph{first-order} effective updates assuming identity Jacobians.

Consider a single update step \(t\). By the Taylor theorem with integral remainder, for any layer \(k\),
\begin{align}
    f_k((W_k + \Delta W_k') z_t) &= f_k(W_k z_t) + \nabla f_k(W_k z_t) \, (\Delta W_k' z_t) \notag \\
    &\quad + \int_0^1 (1-s) \nabla^2 f_k(W_k z_t + s \Delta W_k' z_t) \, ds \, ( \Delta W_k' z_t \otimes \Delta W_k' z_t) \notag \\
    &= f_k(W_k z_t) + \nabla f_k(W_k z_t) \, (\Delta W_k' z_t) + R_k(z_t), 
\end{align}
where \(\otimes\) denotes the appropriate bilinear action of the Hessian.

The remainder \(R_k(z_t)\) satisfies
\begin{align*}
    \|R_k(z_t)\| \leq \frac{H}{2} \|\Delta W_k' z_t\|^2 \leq \frac{H}{2} \xi^2 \|z_t\|^2.
\end{align*}

Substituting into the MoLEx update, we obtain
\begin{align}
    z_{t+1} &= z_t + \alpha \Big[ f_t(W_t z_t) + \nabla f_t(W_t z_t) (\Delta W_t' z_t) + R_t(z_t) \Big] \notag \\
    &\quad + (1-\alpha) \Big[ f_{h_t}(W_{h_t} z_t) + \nabla f_{h_t}(W_{h_t} z_t) (\Delta W_{h_t}' z_t) + R_{h_t}(z_t) \Big].
\end{align}

Recall that the generalize state-based update is presented by
\begin{align*}
    x_{t+1} &= x_t + f_t(W_t x_t) + \alpha \, \Delta W_t x_t + (1-\alpha) \, \Delta W_{h_t} x_t.
\end{align*}

Define the \emph{linearized MoLEx control} at step \(t\) as
\begin{align*}
    \tilde{M}_t := \alpha \, \nabla f_t(W_t z_t) \, \Delta W_t' + (1-\alpha) \, \nabla f_{h_t}(W_{h_t} z_t) \, \Delta W_{h_t}'.
\end{align*}

The MoLEx update expands to
\begin{align}
    z_{t+1} &= z_t + \alpha f_t(W_t z_t) + (1-\alpha) f_{h_t}(W_{h_t} z_t)
    + \tilde{M}_t z_t + \alpha R_t(z_t) + (1-\alpha) R_{h_t}(z_t) \notag \\ 
    &= z_t + f_t(W_t z_t) + (1-\alpha)[f_{h_t}(W_{h_t}z_t) - f_t(W_tz_t)] +  \tilde{M}_t z_t + \alpha R_t(z_t) + (1-\alpha) R_{h_t}(z_t).
\end{align}

Assume that in step $t$, we have $z_t = x_t$, then the difference $z_{t+1} - x_{t+1}$ consists of three terms:
\begin{enumerate}
    \item Base mixing error: $(1-\alpha)[f_{h_t}(W_{h_t} x_t) - f_t(W_t x_t)]$, with $L^2$ norm $\leq |1-\alpha|^2 (\gamma^2 B + \beta^2)$.
    \item Control mismatch: $\tilde{M}_t x_t - (\alpha \Delta W_t + (1-\alpha) \Delta W_{h_t}) x_t$. By linear-case construction and $\|\nabla f_k\| \leq L_f$, this is $O(\xi \sqrt{B})$.
    \item Remainder: $\|\alpha R_t + (1-\alpha) R_{h_t}\| = O(\xi^2 B)$.
\end{enumerate}

Applying the Cauchy-Schwarz inequality $\|\sum_{i=1}^3 a_i\|^2 \leq 3 \sum \|a_i\|^2$ over the three terms, taking expectations, and absorbing higher-order terms yields the per-step bound. 
Then, for each $t$, we have
    \[
    \mathbb{E}\left[ \|z_{t+1} - x_{t+1}\|^2 \right] \leq C \Big( |1-\alpha|^2 (\gamma^2 + \beta^2/B) + \xi^2 \Big),
    \]
    with constant $C > 0$ depending on $\alpha, \xi, L_f, H, B$. As $\xi \to 0$, and $\gamma,\beta \to 0$ (or $|1-\alpha| \to 0$), the trajectories converge uniformly in $L^2$:
    \[
    \max_{t} \mathbb{E}\left[ \|z_t - x_t\|^2 \right] \to 0.
    \]
Accordingly, a generalized state‑based cross‑layer communication mechanism suffices to approximate the MoLEx representation while achieving a markedly reduced computational cost. This completes the proof.
\end{proof}

\subsection{\cref{thm:moc_stability}}

\begin{proof}
We analyze the forward dynamics by bounding the growth of hidden state norms across layers. The key idea is to show that MoC's adaptive routing mechanism can select control modules with smaller operator norms, leading to improved stability.
For Parallel Control with $\alpha = 1$ (no routing), the update is:
\begin{equation}
    x_{t+1}^{c} = x_t^c + f_t(W_tx_t^c) + \Delta W_t x_t^c.
\end{equation}

Taking norms and applying the triangle inequality:
\begin{align}
    \|x_{t+1}^{c}\| &\leq \|x_t^c\| + \|f_t(W_tx_t^c)\| + \|\Delta W_t x_t^c\| \notag \\
    &\leq \|x_t^c\| + L_f\|x_t^c\| + \|\Delta W_t\|\|x_t^c\|.
\end{align}

where we used the assumption that $\|f_t(W_tx_t^c)\| \leq L_f\|x_t^c\|$ (bounded nonlinearity with Lipschitz constant $L_f$).

Using the low-rank factorization $\Delta W_t = B_t A_t$ with $\|B_t\| \leq \beta_t$ and $\|A_t\| \leq \gamma$:
\begin{equation}
    \|x_{t+1}^{c}\| \leq (1 + L_f + \beta_t\gamma)\|x_t^c\|.
\end{equation}

By taking the expectation over data distribution, we have:
\begin{align}
    \mathbb{E}[\|x^c_{t+1}\|/\|x^c_t\|] &\leq (1 + L_f + \beta_t\gamma)
\end{align}

Similarly, for the MoC dynamic with mixing parameter $\alpha \in [0,1]$ and TopK routing, we imply:
\begin{equation}
    x_{t+1} = x_t + f_t(W_tx_t) + \alpha \Delta W_t x_t + (1-\alpha)\sum_{k \in S_K(x_t)} \pi_k(x_t) \Delta W_k x_t
\end{equation}

where $S_K(x_t) \subseteq \{1,\ldots,L\}$ are the top-$K$ selected experts and $\sum_{k \in S_K} \pi_k(x_t) = 1$. Similarly to Parallel Control, we take norms for MoC:
\begin{align}
    \|x_{t+1}\| &\leq \|x_t\| + \|f_t(W_tx_t)\| + \alpha\|\Delta W_t x_t\| + (1-\alpha)\left\|\sum_{k \in S_K} \pi_k(x_t) \Delta W_k x_t\right\| \notag \\
    &\leq \|x_t\| + L_f\|x_t\| + \alpha\beta_t\gamma\|x_t\| + (1-\alpha)\sum_{k \in S_K} \pi_k(x_t)\|\Delta W_k x_t\|
\end{align}

where the last step uses convexity of the norm: $\|\sum_k \pi_k v_k\| \leq \sum_k \pi_k \|v_k\|$ for $\pi_k \geq 0, \sum_k \pi_k = 1$. Using $\|\Delta W_k\| \leq \beta_k\gamma$, we imply:
\begin{align}
    \|x_{t+1}\| &\leq \|x_t\| + L_f\|x_t\| + \alpha\beta_t\gamma\|x_t\| + (1-\alpha)\sum_{k \in S_K} \pi_k(x_t)\beta_k\gamma\|x_t\| \notag \\
    &= \left(1 + L_f + \alpha\beta_t\gamma + (1-\alpha)\beta_K(x_t)\gamma\right)\|x_t\|
\end{align}

where we define the effective routing norm:
\begin{equation}
    \beta_K(x_t) := \sum_{k \in S_K(x_t)} \pi_k(x_t)\beta_k.
\end{equation}

Taking expectations over the data distribution $\mathcal{D}$:
\begin{align}
    \mathbb{E}[\|x_{t+1}\|/\|x_t\|] &\leq \left(1 + L_f + \alpha\beta_t\gamma + (1-\alpha)\mathbb{E}[\beta_K(x_t) | x_t]\gamma\right)
\end{align}

By routing learns to select experts with smaller operator norms on average:
\begin{equation}
    \mathbb{E}[\beta_K(x_t) | x_t] < \beta_t, 
\end{equation}
then the norm bound forward pass will be smaller than Parallel Control:
\begin{equation*}
  (1 + L_f + \alpha\beta_t\gamma + (1-\alpha)\mathbb{E}[\beta_K(x_t)|x_t]\gamma)  < (1 + L_f + \alpha\beta_t\gamma + (1-\alpha)\beta_t\gamma) = (1 + L_f + \beta_t\gamma), \\
\end{equation*}
which imply that MoC achieves improved forward stability compared to conventional control when routing is effective.
\end{proof}

Noting that with the assumption $\mathbb{E}[\beta_K(x_t)|x_t] < \beta_t$, we employ the load‑balancing loss \cite{fedus2022_switch, lepikhin2020_gshard} to encourage balanced token routing across experts. This promotes more effective utilization of high‑capacity experts, thereby increasing the likelihood of selecting suitable experts that meet the assumption.

\subsection{\cref{cor:gradient}}

\begin{proof}
We analyze how control affects gradient flow during backpropagation through the network.

The backward gradient through block $t$ is computed via the chain rule:
\begin{equation}
    \frac{\partial \mathcal{L}}{\partial x_t} = J_t^\top \frac{\partial \mathcal{L}}{\partial x_{t+1}},
\end{equation}
then for MoC, differentiating the update equation gives:
\begin{align}
    J_t &= \frac{\partial}{\partial x_t}\left[x_t + f_t(W_t x_t) + \alpha\Delta W_t x_t + (1-\alpha)\sum_{k \in S_K} \pi_k(x_t) \Delta W_k x_t\right] \notag \\
    &= I + \frac{\partial f_t(W_t x_t)}{\partial x_t} + \alpha\Delta W_t + (1-\alpha)\sum_{k \in S_K} \pi_k(x_t)\Delta W_k \notag \\
    &\quad + (1-\alpha)\sum_{k \in S_K} \frac{\partial \pi_k(x_t)}{\partial x_t} \otimes (\Delta W_k x_t)
\end{align}

For simplicity, we focus on the dominant terms and assume routing weights are approximately constant locally (i.e., the gradient w.r.t. $\pi_k$ is small), giving:
\begin{equation}
    J_t \approx I + \frac{\partial f_t(W_t x_t)}{\partial x_t} + \alpha\Delta W_t + (1-\alpha)\sum_{k \in S_K} \pi_k(x_t)\Delta W_k
\end{equation}

Then we take norms and using the Lipschitz property of $f_t$ with constant $L_f$:
\begin{align}
    \|J_t\| &\leq 1 + \left\|\frac{\partial f_t(W_t x_t)}{\partial x_t}\right\| + \alpha\|\Delta W_t\| + (1-\alpha)\left\|\sum_{k \in S_K} \pi_k(x_t)\Delta W_k\right\| \notag \\
    &\leq 1 + L_f\|W_t\| + \alpha\|\Delta W_t\| + (1-\alpha)\sum_{k \in S_K} \pi_k(x_t)\|\Delta W_k\| \notag\\
    &\leq 1 + L_f\|W_t\| + \alpha\beta_t\gamma + (1-\alpha)\beta_K(x_t)\gamma
\end{align}

Under the assumptions of \cref{thm:moc_stability}, we obtain:
\begin{equation*}
    1 + L_f\|W_t\| + \alpha\beta_t\gamma + (1-\alpha)\beta_K(x_t)\gamma < 1 + L_f\|W_t\| +\beta_t\gamma, 
\end{equation*}
which indicates that the Jacobian norm of MoC is more tightly bounded than that of Parallel Control, leading to a more stable gradient flow.
\end{proof}
In addition to the stability analyses of the representation and its gradient dynamics, we assess the fidelity of the learned representations. Assume $\Delta W^* x$ denote the optimal control signal that yields task-optimal representations. Given a MoC signal $g_\text{moc}(x)$, then the expected control-approximation error over data distribution is defined as follows:
\begin{equation}
    \bar{\epsilon}^2 \coloneqq \mathbb{E}\left[\|g_{\text{moc}}(x) - \Delta W^*x\|^2\right].
\end{equation}

\subsection{\cref{lem:representation_error}}

\begin{proof}
Given input $x$, then control error of Parallel Control is:
\begin{align}
    \delta_t^c(x) &= \Delta W_t x - \Delta W^*x,
\end{align}
then control update at block $t$ of MoC can be written as follows:
\begin{equation}
    g_{\text{moc}}(x) = \alpha \Delta W_t x + (1-\alpha) \sum_{k\in S_K(x)}\pi_k(x) \Delta W_k x.
\end{equation}
Next, the control representation error of MoC is:
\begin{align}
    \delta_t(x) &= g_{\text{moc}}(x) - \Delta W^*x \notag \\
    &= \alpha[\Delta W_t x - \Delta W^*x] + (1-\alpha)\sum_{k\in S_K(x)}\pi_k(x)[\Delta W_k x - \Delta W^*x] \\
    &= \alpha \delta_t^c(x) + (1-\alpha)\sum_{k\in S_K(x)}\pi_k(x)\delta_k^c(x),
\end{align}
where we have $\sum_{k\in S_K(x)}\pi_k(x) = 1$, $\pi_k(x) \ge 0, \forall \, k = \overline{1, L}$. Similarly, 

By convexity of $\|\cdot\|^2$ (Jensen's inequality):
\begin{align}
    \|\delta_t(x)\|^2 &= \left\|\alpha\delta^c_t(x) + (1-\alpha)\sum_{k \in S_K(x)}\pi_k(x)\delta^c_k(x)\right\|^2 \notag \\
    &\leq \alpha\|\delta_{t}^c(x)\|^2 + (1-\alpha)\left\|\sum_{k \in S_K(x)}\pi_k(x)\delta_k^c(x)\right\|^2 \\
\end{align}
Taking expectation over $x\sim\mathcal{D}$, $t =\overline{1, L}$:
\begin{align}
    \bar{\epsilon}^2 &= \mathbb{E}[\|\delta_t(x)\|^2] \notag \\
    &\leq \alpha \mathbb{E}[\|\delta^c_t(x)\|^2] + (1-\alpha)\mathbb{E}\left[\left\|\sum_{k \in S_K(x)}\pi_k(x)\delta_k^c(x)\right\|^2\right] \notag \\
    &= \alpha \bar{\epsilon}_{c}^2 + (1-\alpha)\bar{\epsilon}_K^2
\end{align}
This completes the proof. 
\end{proof}
\section{Experimental Details}
\label{appendix:exp_details}

This section provides comprehensive details about our experimental setup, including hyperparameter configurations, dataset specifications, and implementation details for all models evaluated in this work.

\subsection{Dataset Details}
\label{subsec:datasets}
To comprehensively evaluate the effectiveness of our proposed method, we conduct experiments across multiple benchmark datasets covering different NLP task categories.
\paragraph{Natural Language Understanding (NLU)}
For NLU tasks, we employ the GLUE benchmark \cite{wang2018glue}, a well-established benchmark suite that includes datasets such as SST-2, MRPC,
CoLA, QNLI, RTE, and STS-B. These tasks collectively evaluate model performance on grammaticality judgment, sentiment classification, semantic similarity, and textual entailment across domain distributions.

\paragraph{Natural Language Generation and Commonsense Reasoning (NLG)}
For evaluating NLG and commonsense reasoning capabilities, we utilize eight benchmark datasets: BoolQ~\citep{clark2019boolq}, PIQA~\citep{bisk2020piqa}, Social IQa~\citep{sap-etal-2019-social}, HellaSwag~\citep{zellers2019hellaswag}, WinoGrande~\citep{sakaguchi2021winogrande}, ARC-Easy/ Challenge~\citep{clark2018think}, and OpenBookQA~\citep{mihaylov2018can}. These datasets assess model understanding of complex reasoning, physical commonsense, and social dynamics.

\subsection{Model Architecture and Training Configuration}

\paragraph{Encoder-Only Models.} We utilized some Encoder-only models as the model backbones, particularly RoBERTa \cite{liu2019roberta} and DeBERTa \cite{he2020deberta}, on different model size (base and large), incorporating LoRA, DoRA, MoLEx and Parallel Control methods as the baselines. The detailed hyperparameter configurations for BERT-based models are illustrated in \cref{tab:bert_hyperparameters}. 

\paragraph{Decoder-Only Models.} We utilized some Large scale models as the backbones, including LLaMA2-7B \cite{touvron2023llama}, LLaMA3-8B \cite{grattafiori2024llama}, Mistral-7B Instruct \cite{jiang2023mistral7b}, Qwen2.5-14B Instruct \cite{qwen2025qwen25technicalreport}. These pretrained models are also finetuned across our methods as well as the baselines. The detailed hyperparameter configurations for the decoder-only models are illustrated in \cref{tab:llama_hyperparameters}. 
For routing, we consider two variants: (1) a linear router and (2) a cosine router with minimal extra parameters via a low-dimensional projection, which expands the latent space for richer representations.

\paragraph{Vision Tasks.} We finetune ViT-Tiny/Base backbones on VTAB datasets, under batch size 128, 50 epochs, learning rate 0.005 to ensure fair comparison.

\paragraph{Weight-based Configurations.} Across weight-based methods, we apply rank-$32$ LoRA modules to $W_Q, W_K$ and $W_V$ in attention and $W_{\text{up}}, W_{\text{down}}$ in the FFN. For MoLEx, as it additionally introduces a lightweight, learnable, routing gate, we use Top-$K$=1, the best-performing configuration reported for MoLEx, and keep the remaining training settings identical across methods and modalities as shown in \Cref{tab:bert_hyperparameters} and \Cref{tab:llama_hyperparameters}. Because MoLEx may exceeds the memory limit, in this case, we distribute it across two GPUs H100 80GB rather than reducing its batch size or rank for fair comparison.

\paragraph{MoC Configurations.}

For efficient Transformer training, we adopt batch-level routing in the MoC framework. Instead of routing each token independently, which would result in token-specific control assignments and hinder efficient batched computation, we make a unified routing decision for all tokens in a batch. To obtain the batch-level TopK controls, we aggregate token-level gate outputs using either $``$mode$"$ or $``$mean$"$ aggregation. The mode strategy selects the most frequently activated controls across tokens, while the mean strategy averages token-level probability vectors and chooses the controls with the highest averaged probabilities. In our experiments, we use either mode- or mean-based batch routing. The load-balancing weight is selected from $[0.001, 0.02]$, and the mixing coefficient is set to $\alpha \in \{0.6, 0.7, \ldots, 0.9, 0.95\}$, following \cref{prop4.1:Expressive_moc}.

\begin{table}[h]
\centering
\caption{Hyperparameter configurations for DeBERTa and RoBERTa models on GLUE tasks. Common settings include AdamW optimizer, linear learning rate schedule, warmup ratio of 0.06, and sequence length of 128 tokens.}
\scalebox{0.9}{
\begin{tabular}{ll|cccccc}
\hline
\textbf{Model} & \textbf{Dataset} & \textbf{SST-2} & \textbf{MRPC} & \textbf{CoLA} & \textbf{QNLI} & \textbf{RTE} & \textbf{STS-B} \\
\hline
& Optimizer & \multicolumn{6}{c}{AdamW} \\
& WarmUp Ratio & \multicolumn{6}{c}{0.06} \\
& LR Schedule & \multicolumn{6}{c}{Linear} \\
& Max Seq Length & \multicolumn{6}{c}{128} \\
\hline
\multirow{3}{*}{DeBERTa$_{\text{base}}$} 
& Batch Size    & 128 & 128 & 64 & 256 & 128 & 128 \\
& Epochs        & 50  & 30  & 80 & 25  & 80  & 40  \\
& Learning Rate & 5e-4 & 4e-4 & 3e-4 & 4e-4 & 4e-4 & 4e-4 \\
\hline
\multirow{3}{*}{DeBERTa$_{\text{large}}$} 
& Batch Size    & 64 & 32 & 32 & 32 & 64 & 32 \\
& Epochs        & 20 & 30 & 20 & 10 & 20 & 10 \\
& Learning Rate & 4e-4 & 3e-4 & 5e-4 & 3e-4 & 4e-4 & 3e-4 \\
\hline
\multirow{3}{*}{RoBERTa$_{\text{base}}$} 
& Batch Size    & 128 & 128 & 64 & 256 & 128 & 128 \\
& Epochs        & 50  & 30  & 80 & 25  & 80  & 40  \\
& Learning Rate & 5e-4 & 4e-4 & 3e-4 & 4e-4 & 4e-4 & 4e-4 \\
\hline
\multirow{3}{*}{RoBERTa$_{\text{large}}$} 
& Batch Size    & 64 & 32 & 32 & 32 & 64 & 32 \\
& Epochs        & 20 & 30 & 20 & 10 & 20 & 10 \\
& Learning Rate & 4e-4 & 3e-4 & 5e-4 & 3e-4 & 4e-4 & 3e-4 \\
\hline
\end{tabular}
}
\label{tab:bert_hyperparameters}
\end{table}

\begin{table}[t]
\centering
\caption{Hyperparameter settings and Mixture-of-Control configuration for LLaMA2-7B, LLaMA3-8B, Mistralv3-7B, and Qwen2.5-14B. Settings are chosen to balance performance and training efficiency.}
\resizebox{0.9\linewidth}{!}{
\begin{tabular}{lcccc}
\toprule
\textbf{Parameter} & \textbf{LLaMA2-7B} & \textbf{LLaMA3-8B} & \textbf{Mistralv3-7B} & \textbf{Qwen2.5-14B} \\
\midrule
Learning rate & $3 \times 10^{-4}$ & $3 \times 10^{-4}$ & $3 \times 10^{-4}$ & $3 \times 10^{-4}$ \\
Batch size & 16 & 16 & 16 & 12 \\
Gradient accumulation steps & 2 & 2 & 2 & 2 \\
Max sequence length & 256 & 256 & 256 & 256 \\
Projection dimension & 256 & 256 & 256 & 256 \\
Training epochs & 3 & 3 & 3 & 3 \\
Weight decay & 0.01 & 0.01 & 0.01 & 0.01 \\
\bottomrule
\end{tabular}
}
\label{tab:llama_hyperparameters}
\end{table}

\paragraph{Compute Resources} The compute resource used in our experiments consists of 8 NVIDIA H100 GPUs, each with 80GB of memory.

\subsection{Additional Experiments}
\label{app:additional_experiments}
\begin{table}[t]
\centering
\caption{GLUE results with DeBERTa-Base/Large backbones. Values are medians of five runs (different seeds); * marks previously published numbers. Best results are in \textbf{bold}.}
\vspace{0.5em}
\label{tab:glue_deberta}

\renewcommand\tabcolsep{5pt}
\renewcommand\arraystretch{1.2}

\resizebox{0.8\linewidth}{!}{
\begin{tabular}{c|l|ccccccc}
\Xhline{1.2pt}
\rowcolor{cyan!20}
\textbf{Model} & \textbf{Method} 
& \textbf{SST-2} & \textbf{MRPC} & \textbf{CoLA} 
& \textbf{QNLI} & \textbf{RTE} & \textbf{STS-B} & \textbf{Avg} \\
\Xhline{1.2pt}

\multirow{5}{*}{\rotatebox{0}{\rotatebox{0}{DeBERTa-Base}}}

& LoRA* 
& 96.2 & 89.4 & 70.4 & 94.3 & 83.4 & 91.0 & 87.5 \\

\cline{2-9}
\noalign{\vskip 0.1mm}

& \cellcolor{gray!10}DoRA* 
& \cellcolor{gray!10}96.0\dn{0.2} 
& \cellcolor{gray!10}90.2\up{0.8} 
& \cellcolor{gray!10}71.2\up{0.8} 
& \cellcolor{gray!10}94.3\up{0.0} 
& \cellcolor{gray!10}84.0\up{0.6} 
& \cellcolor{gray!10}91.0\up{0.0} 
& \cellcolor{gray!10}87.8\up{0.3} \\

& MoLEx 
& 96.0\dn{0.2} 
& 90.1\up{0.7} 
& 71.0\up{0.6} 
& 94.3\up{0.0} 
& 85.8\up{2.4} 
& 91.4\up{0.4} 
& 88.1\up{0.6} \\

& \cellcolor{gray!10}Control 
& \cellcolor{gray!10}96.3\up{0.1} 
& \cellcolor{gray!10}89.8\up{0.4} 
& \cellcolor{gray!10}71.3\up{0.9} 
& \cellcolor{gray!10}94.4\up{0.1} 
& \cellcolor{gray!10}84.8\up{1.4} 
& \cellcolor{gray!10}91.5\up{0.5} 
& \cellcolor{gray!10}88.0\up{0.5} \\

& \cellcolor{moccolor}\textbf{MoC} 
& \cellcolor{moccolor}\textbf{96.9}\up{0.7} 
& \cellcolor{moccolor}\textbf{91.0}\up{1.6} 
& \cellcolor{moccolor}\textbf{71.8}\up{1.4} 
& \cellcolor{moccolor}\textbf{94.5}\up{0.2} 
& \cellcolor{moccolor}\textbf{87.8}\up{4.4} 
& \cellcolor{moccolor}\textbf{91.9}\up{0.9} 
& \cellcolor{moccolor}\textbf{89.0}\up{1.5} \\

\hline

\multirow{5}{*}{\rotatebox{0}{DeBERTa-Large}}

& LoRA* 
& 96.9 & 91.8 & 75.0 & 95.9 & 91.0 & 92.7 & 90.5 \\

\cline{2-9}
\noalign{\vskip 0.1mm}

& \cellcolor{gray!10}DoRA* 
& \cellcolor{gray!10}96.6\dn{0.3} 
& \cellcolor{gray!10}92.1\up{0.3} 
& \cellcolor{gray!10}76.9\up{1.9} 
& \cellcolor{gray!10}95.8\dn{0.1} 
& \cellcolor{gray!10}91.1\up{0.1} 
& \cellcolor{gray!10}92.6\dn{0.1} 
& \cellcolor{gray!10}90.9\up{0.4} \\

& MoLEx 
& 96.6\dn{0.3} 
& 91.9\up{0.1} 
& 75.9\up{0.9} 
& 95.9\up{0.0} 
& 91.6\up{0.6} 
& 92.7\up{0.0} 
& 90.8\up{0.3} \\

& \cellcolor{gray!10}Control 
& \cellcolor{gray!10}96.7\dn{0.2} 
& \cellcolor{gray!10}\textbf{92.7}\up{0.9} 
& \cellcolor{gray!10}76.4\up{1.4} 
& \cellcolor{gray!10}95.9\up{0.0} 
& \cellcolor{gray!10}90.0\dn{1.0} 
& \cellcolor{gray!10}92.3\dn{0.4} 
& \cellcolor{gray!10}90.6\up{0.1} \\

& \cellcolor{moccolor}\textbf{MoC} 
& \cellcolor{moccolor}\textbf{97.3}\up{0.4} 
& \cellcolor{moccolor}\textbf{92.7}\up{0.9} 
& \cellcolor{moccolor}\textbf{77.0}\up{2.0} 
& \cellcolor{moccolor}\textbf{96.2}\up{0.3} 
& \cellcolor{moccolor}\textbf{93.9}\up{2.9} 
& \cellcolor{moccolor}\textbf{92.9}\up{0.2} 
& \cellcolor{moccolor}\textbf{91.7}\up{1.2} \\

\Xhline{1.2pt}
\end{tabular}
}
\end{table}

\begin{table*}[ht]
\centering
\caption{\texttt{VTAB-1K} results evaluated on Top-1 accuracy. All methods are applied to finetune the same set of LoRA parameters on ViT-Tiny/Base pre-trained with \texttt{ImageNet-21K} dataset.}
\renewcommand\arraystretch{1.5}
\renewcommand\tabcolsep{4pt}
\resizebox{\linewidth}{!}{
\begin{tabular}{cccccccc|cccc|cccccccc|c}
\toprule
\cellcolor[HTML]{FFFFFF}&
 \multicolumn{7}{c|}{\textbf{Natural}}&
 \multicolumn{4}{c|}{\textbf{Specialized}}&
 \multicolumn{8}{c|}{\textbf{Structured}}&
  \cellcolor[HTML]{FFFFFF} \\ 
\cmidrule(lr){2-8} \cmidrule(lr){9-12} \cmidrule(lr){13-20}
\rowcolor[HTML]{FFFFFF} 
\multicolumn{1}{c}{\cellcolor[HTML]{FFFFFF}\textbf{Method}} &
  \cellcolor[HTML]{FFFFFF}\rot{Cifar100} &
  \cellcolor[HTML]{FFFFFF}\rot{Caltech101} &
  \cellcolor[HTML]{FFFFFF}\rot{DTD} &
  \cellcolor[HTML]{FFFFFF}\rot{Flower102} &
  \cellcolor[HTML]{FFFFFF}\rot{Pets} &
  \cellcolor[HTML]{FFFFFF}\rot{SVHN} &
  \cellcolor[HTML]{FFFFFF}\rot{Sun397} &
  \cellcolor[HTML]{FFFFFF}\rot{Camelyon16} &
  \cellcolor[HTML]{FFFFFF}\rot{EuroSAT} &
  \cellcolor[HTML]{FFFFFF}\rot{Resisc45} &
  \cellcolor[HTML]{FFFFFF}\rot{Retinopathy} &
  \cellcolor[HTML]{FFFFFF}\rot{Clevr-Count} &
  \cellcolor[HTML]{FFFFFF}\rot{Clevr-Dist} &
  \cellcolor[HTML]{FFFFFF}\rot{DMLab} &
  \cellcolor[HTML]{FFFFFF}\rot{KITTI} &
  \cellcolor[HTML]{FFFFFF}\rot{dSpr-Loc} &
  \cellcolor[HTML]{FFFFFF}\rot{dSpr-Ori} &
  \cellcolor[HTML]{FFFFFF}\rot{sNORB-Azi} &
  \cellcolor[HTML]{FFFFFF}\rot{sNORB-Ele} &
  \cellcolor[HTML]{FFFFFF}\textbf{Mean} \\ 
\toprule

\multicolumn{21}{c}{\textbf{ViT-Tiny}} \\
\midrule
\cellcolor[HTML]{FFFFFF}{LoRA}
& 34.0 & 88.0 & \textbf{61.0} & 93.0 & 85.0 & 86.5 & 25.0 & 88.0 & \textbf{94.5} & 78.0 & \textbf{78.0} & 85.0 & \textbf{62.0} & \textbf{45.0} & 82.0 & 36.5 & 34.0 & 15.5 & \textbf{39.0} & 63.7 \\

\rowcolor{gray!10}
\cellcolor[HTML]{FFFFFF}{DoRA}
& 33.0\dn{1.0} & \textbf{89.0}\up{1.0} & 59.5\dn{1.5} & 92.5\dn{0.5} & 84.0\dn{1.0} & 86.5\up{0.0} & 24.5\dn{0.5} & 86.0\dn{2.0} & 94.0\dn{0.5} & 79.0\up{1.0} & \textbf{78.0}\up{0.0} & \textbf{85.5}\up{0.5} & 61.0\dn{1.0} & 44.0\dn{1.0} & \textbf{83.0}\up{1.0} & 36.0\dn{0.5} & \textbf{35.5}\up{1.5} & 15.0\dn{0.5} & \textbf{39.0}\up{0.0} & 63.4\dn{0.3} \\

\cellcolor[HTML]{FFFFFF}{MoLEx}
& 38.0\up{4.0} & 88.0\up{0.0} & 59.0\dn{2.0} & 94.5\up{1.5} & 84.0\dn{1.0} & 87.5\up{1.0} & 25.5\up{0.5} & 88.0\up{0.0} & 93.5\dn{1.0} & 79.5\up{1.5} & 76.5\dn{1.5} & 84.5\dn{0.5} & 60.0\dn{2.0} & 43.0\dn{2.0} & 82.0\up{0.0} & 36.5\up{0.0} & 35.0\up{1.0} & 16.0\up{0.5} & 38.0\dn{1.0} & 63.6\dn{0.1} \\

\rowcolor{gray!10}
\cellcolor[HTML]{FFFFFF}{Control}
& 33.5\dn{0.5} & 84.0\dn{4.0} & 56.5\dn{4.5} & 92.5\dn{0.5} & 84.5\dn{0.5} & 82.5\dn{4.0} & 26.5\up{1.5} & 87.0\dn{1.0} & 93.0\dn{1.5} & 79.5\up{1.5} & \textbf{78.0}\up{0.0} & 77.0\dn{8.0} & 61.0\dn{1.0} & 43.5\dn{1.5} & 79.0\dn{3.0} & 32.0\dn{4.5} & 30.0\dn{4.0} & \textbf{18.0}\up{2.5} & 36.0\dn{3.0} & 61.8\dn{1.9} \\

\cellcolor[HTML]{FFFFFF}\textbf{MoC}
& \textbf{39.0}\up{5.0} & \textbf{89.0}\up{1.0} & 59.5\dn{1.5} & \textbf{95.5}\up{2.5} & \textbf{86.5}\up{1.5} & \textbf{89.0}\up{2.5} & \textbf{31.5}\up{6.5} & \textbf{89.0}\up{1.0} & 94.0\dn{0.5} & \textbf{82.5}\up{4.5} & \textbf{78.0}\up{0.0} & \textbf{85.5}\up{0.5} & 61.0\dn{1.0} & 43.5\dn{1.5} & \textbf{83.0}\up{1.0} & \textbf{37.0}\up{0.5} & \textbf{35.5}\up{1.5} & 17.0\up{1.5} & \textbf{39.0}\up{0.0} & \textbf{65.0}\up{1.3} \\

\midrule
\multicolumn{21}{c}{\textbf{ViT-Base}} \\
\midrule
\cellcolor[HTML]{FFFFFF}{LoRA}
& 60.0 & 96.0 & \textbf{66.0} & \textbf{99.5} & \textbf{94.0} & 83.5 & 46.5 & 90.0 & 95.0 & 82.5 & 76.0 & \textbf{82.0} & 55.0 & 44.5 & 80.0 & 35.5 & 33.0 & 15.5 & \textbf{39.0} & 67.0 \\

\rowcolor{gray!10}
\cellcolor[HTML]{FFFFFF}{DoRA}
& 60.5\up{0.5} & 96.0\up{0.0} & \textbf{66.0}\up{0.0} & \textbf{99.5}\up{0.0} & 93.0\dn{1.0} & 84.5\up{1.0} & 47.5\up{1.0} & 90.0\up{0.0} & 93.5\dn{1.5} & 81.5\dn{1.0} & 75.5\dn{0.5} & 76.0\dn{6.0} & 56.0\up{1.0} & 43.5\dn{1.0} & 80.5\up{0.5} & 37.0\up{1.5} & 33.5\up{0.5} & 16.0\up{0.5} & 34.5\dn{4.5} & 66.6\dn{0.4} \\

\cellcolor[HTML]{FFFFFF}{MoLEx}
& 61.0\up{1.0} & 95.5\dn{0.5} & 64.5\dn{1.5} & \textbf{99.5}\up{0.0} & 93.0\dn{1.0} & 85.5\up{2.0} & 47.5\up{1.0} & 90.0\up{0.0} & 95.5\up{0.5} & 83.0\up{0.5} & 79.0\up{3.0} & 75.5\dn{6.5} & 57.5\up{2.5} & 44.0\dn{0.5} & 81.5\up{1.5} & 40.0\up{4.5} & 33.0\up{0.0} & 14.0\dn{1.5} & 34.5\dn{4.5} & 67.1\up{0.1} \\

\rowcolor{gray!10}
\cellcolor[HTML]{FFFFFF}{Control}
& 59.5\dn{0.5} & \textbf{96.5}\up{0.5} & 64.5\dn{1.5} & 99.0\dn{0.5} & 92.5\dn{1.5} & 86.0\up{2.5} & 45.5\dn{1.0} & 88.0\dn{2.0} & \textbf{96.0}\up{1.0} & \textbf{84.0}\up{1.5} & 78.0\up{2.0} & 78.5\dn{3.5} & 58.0\up{3.0} & 42.0\dn{2.5} & \textbf{84.0}\up{4.0} & 40.5\up{5.0} & \textbf{34.0}\up{1.0} & \textbf{16.5}\up{1.0} & 34.5\dn{4.5} & 67.2\up{0.2} \\

\cellcolor[HTML]{FFFFFF}\textbf{MoC}
& \textbf{63.0}\up{3.0} & \textbf{96.5}\up{0.5} & 64.5\dn{1.5} & \textbf{99.5}\up{0.0} & 93.5\dn{0.5} & \textbf{86.5}\up{3.0} & \textbf{49.5}\up{3.0} & \textbf{91.0}\up{1.0} & \textbf{96.0}\up{1.0} & \textbf{84.0}\up{1.5} & \textbf{79.5}\up{3.5} & 77.0\dn{5.0} & \textbf{60.5}\up{5.5} & \textbf{51.0}\up{6.5} & 82.0\up{2.0} & \textbf{41.0}\up{5.5} & 33.0\up{0.0} & 14.5\dn{1.0} & 36.5\dn{2.5} & \textbf{68.4}\up{1.4} \\

\bottomrule
\end{tabular}
}
\label{table:vit_detailed_accuracy}
\end{table*}

\begin{figure}[t]
    \centering
    \includegraphics[width=\linewidth]{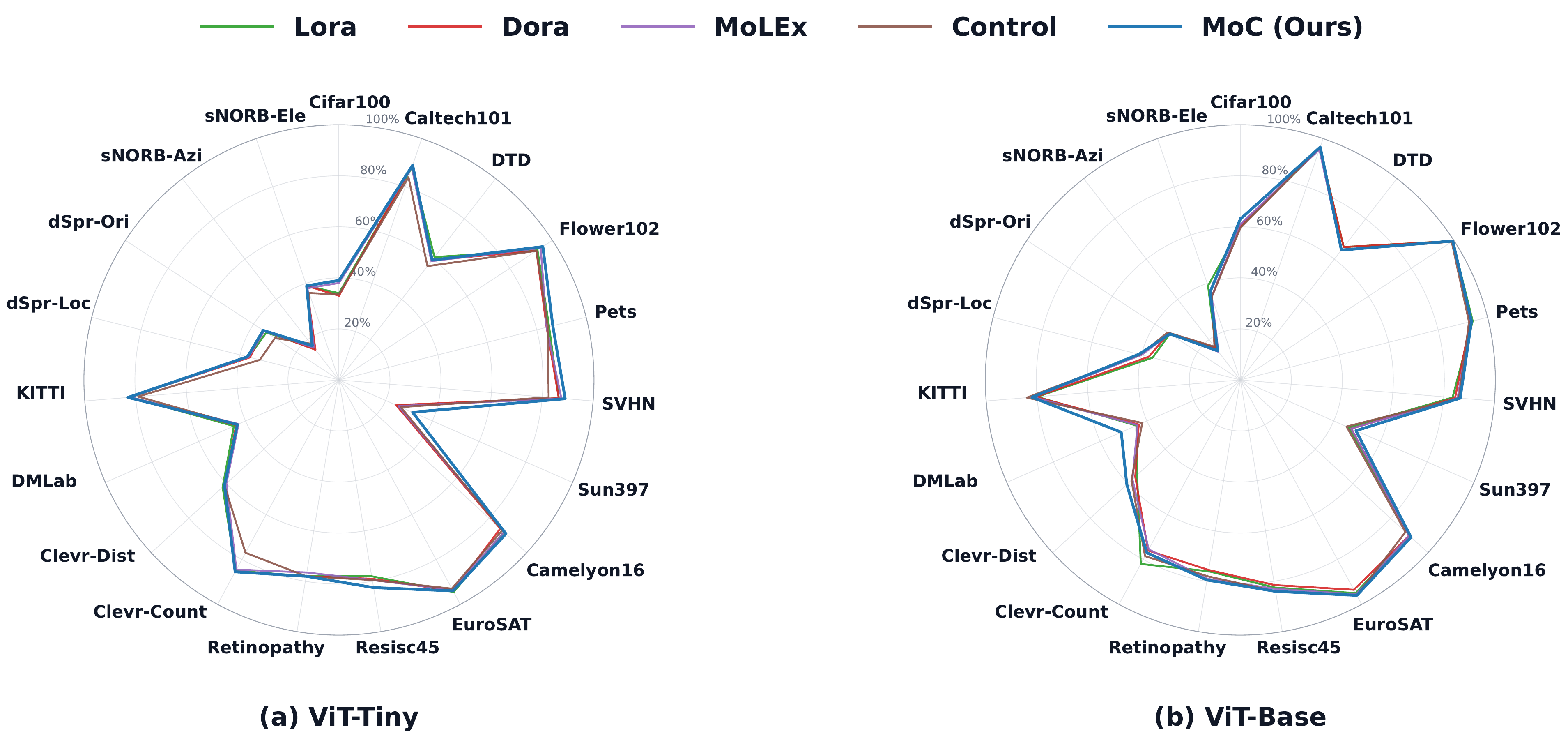}
\caption{Benchmark comparison of LoRA, DoRA, MoLEx, Control, and MoC on VTAB across 19 datasets, shown separately for ViT-Tiny (left) and ViT-Base (right). Values indicate classification accuracy, with the outer ring corresponding to 100\%.}
\label{fig:vit-radar-combined}
\end{figure}

\paragraph{Efficiency Comparison.}
As shown in \Cref{tab:main_llm}, we compare MoC with recent strong baselines selected primarily for competitive performance rather than identical parameter budgets; thus, these results do not fully characterize parameter efficiency. To provide a fairer comparison, \Cref{tab:commonsense_mistral_qwen} evaluates representative methods under matched parameter budgets on Mistralv3-7B-Instruct and Qwen2.5-14B-Instruct. Under this setting, MoC maintains memory usage close to LoRA while improving average accuracy by \textbf{8.0\%} and \textbf{3.2\%}, respectively. Compared with Control, MoC introduces only modest routing overhead while enabling cross-layer communication. It also avoids the additional pretrained-block forward passes required by full-block communication methods such as MoLEx. On Qwen2.5-14B-Instruct, MoC uses nearly half the training memory of MoLEx while achieving \textbf{2.5\%} higher average accuracy.

\paragraph{Batch Routing Scheme.}
As discussed in MoC Configurations, we aggregate token-level gate outputs using either $``$mode$"$ or $``$mean$"$ aggregation. Here we further provide a comparison of batch-level routing schemes in \Cref{tab:aggregation_scheme}. Mean aggregation improves average accuracy over mode aggregation by \textbf{0.3\%} on Mistralv3-7B-Instruct and \textbf{1.2\%} on Qwen2.5-14B-Instruct, indicating moderate model-dependent sensitivity to the routing scheme.

\begin{table*}[t]
\centering
\caption{Results on commonsense reasoning benchmarks with Mistralv3-7B-Instruct and Qwen2.5-14B-Instruct backbones. 
Performance improvements and degradations are measured relative to the LoRA baseline for each backbone. 
Best task performance is shown in \textbf{bold}.}
\vspace{0.5em}
\label{tab:commonsense_mistral_qwen}

\renewcommand\tabcolsep{3.5pt}
\renewcommand\arraystretch{1.2}

\resizebox{0.98\linewidth}{!}{
\begin{tabular}{c|l|c|cccccccc|c|cc}
\Xhline{1.2pt}
\rowcolor{cyan!20}
\textbf{Model} & \textbf{Method} 
& \textbf{Trainable (\%)} 
& \textbf{BoolQ} & \textbf{PIQA} & \textbf{SIQA}
& \textbf{HellaSwag} & \textbf{WinoGrande}
& \textbf{ARC-e} & \textbf{ARC-c} & \textbf{OBQA}
& \textbf{Avg}
& \textbf{Time (h)} & \textbf{Memory (GB)} \\
\Xhline{1.2pt}

\multirow{5}{*}{Mistralv3-7B}

& LoRA
& 0.49
& 67.8
& 83.0
& 77.8
& 89.4
& 79.3
& 78.7
& 64.9
& 79.2
& 77.5
& 5.8
& 54.3 \\

\cline{2-14}
\noalign{\vskip 0.1mm}

& \cellcolor{gray!10}DoRA
& \cellcolor{gray!10}0.50
& \cellcolor{gray!10}68.1\up{0.3}
& \cellcolor{gray!10}82.4\dn{0.6}
& \cellcolor{gray!10}77.7\dn{0.1}
& \cellcolor{gray!10}88.1\dn{1.3}
& \cellcolor{gray!10}78.7\dn{0.6}
& \cellcolor{gray!10}77.1\dn{1.6}
& \cellcolor{gray!10}64.3\dn{0.6}
& \cellcolor{gray!10}79.4\up{0.2}
& \cellcolor{gray!10}77.0\dn{0.5}
& \cellcolor{gray!10}9.5
& \cellcolor{gray!10}64.7 \\

& MoLEx
& 0.49
& 69.5\up{1.7}
& 83.5\up{0.5}
& 77.9\up{0.1}
& 89.4\up{0.0}
& 82.2\up{2.9}
& 80.1\up{1.4}
& 67.6\up{2.7}
& 81.0\up{1.8}
& 78.9\up{1.4}
& 10.0
& 75.4 \\

& \cellcolor{gray!10}Control
& \cellcolor{gray!10}0.46
& \cellcolor{gray!10}70.6\up{2.8}
& \cellcolor{gray!10}86.9\up{3.9}
& \cellcolor{gray!10}79.0\up{1.2}
& \cellcolor{gray!10}94.2\up{4.8}
& \cellcolor{gray!10}84.8\up{5.5}
& \cellcolor{gray!10}87.9\up{9.2}
& \cellcolor{gray!10}76.5\up{11.6}
& \cellcolor{gray!10}84.8\up{5.6}
& \cellcolor{gray!10}83.1\up{5.6}
& \cellcolor{gray!10}\textbf{4.0}
& \cellcolor{gray!10}\textbf{53.2} \\

& \cellcolor{moccolor}\textbf{MoC}
& \cellcolor{moccolor}0.48
& \cellcolor{moccolor}\textbf{74.2}\up{6.4}
& \cellcolor{moccolor}\textbf{88.9}\up{5.9}
& \cellcolor{moccolor}\textbf{81.0}\up{3.2}
& \cellcolor{moccolor}\textbf{95.6}\up{6.2}
& \cellcolor{moccolor}\textbf{87.9}\up{8.6}
& \cellcolor{moccolor}\textbf{88.9}\up{10.2}
& \cellcolor{moccolor}\textbf{79.5}\up{14.6}
& \cellcolor{moccolor}\textbf{87.8}\up{8.6}
& \cellcolor{moccolor}\textbf{85.5}\up{8.0}
& \cellcolor{moccolor}6.0
& \cellcolor{moccolor}55.6 \\

\hline

\multirow{5}{*}{Qwen2.5-14B}

& LoRA
& 0.43
& 72.9
& 88.0
& 80.7
& 94.4
& 87.3
& 92.8
& 85.8
& 93.6
& 86.9
& 7.9
& 60.9 \\

\cline{2-14}
\noalign{\vskip 0.1mm}

& \cellcolor{gray!10}DoRA
& \cellcolor{gray!10}0.44
& \cellcolor{gray!10}73.8\up{0.9}
& \cellcolor{gray!10}88.8\up{0.8}
& \cellcolor{gray!10}81.5\up{0.8}
& \cellcolor{gray!10}95.6\up{1.2}
& \cellcolor{gray!10}88.2\up{0.9}
& \cellcolor{gray!10}94.3\up{1.5}
& \cellcolor{gray!10}86.6\up{0.8}
& \cellcolor{gray!10}89.0\dn{4.6}
& \cellcolor{gray!10}87.2\up{0.3}
& \cellcolor{gray!10}14.0
& \cellcolor{gray!10}90.4 \\

& MoLEx
& 0.43
& 73.4\up{0.5}
& 89.8\up{1.8}
& 81.7\up{1.0}
& 95.3\up{0.9}
& 88.2\up{0.9}
& 94.1\up{1.3}
& 86.7\up{0.9}
& 91.2\dn{2.4}
& 87.6\up{0.7}
& 19.4
& 116.3 \\

& \cellcolor{gray!10}Control
& \cellcolor{gray!10}0.43
& \cellcolor{gray!10}74.2\up{1.3}
& \cellcolor{gray!10}91.4\up{3.4}
& \cellcolor{gray!10}82.0\up{1.3}
& \cellcolor{gray!10}96.4\up{2.0}
& \cellcolor{gray!10}87.7\up{0.4}
& \cellcolor{gray!10}\textbf{96.7}\up{3.9}
& \cellcolor{gray!10}90.0\up{4.2}
& \cellcolor{gray!10}95.4\up{1.8}
& \cellcolor{gray!10}89.2\up{2.3}
& \cellcolor{gray!10}\textbf{6.5}
& \cellcolor{gray!10}\textbf{58.7} \\

& \cellcolor{moccolor}\textbf{MoC}
& \cellcolor{moccolor}0.43
& \cellcolor{moccolor}\textbf{75.5}\up{2.6}
& \cellcolor{moccolor}\textbf{91.8}\up{3.8}
& \cellcolor{moccolor}\textbf{82.4}\up{1.7}
& \cellcolor{moccolor}\textbf{96.8}\up{2.4}
& \cellcolor{moccolor}\textbf{91.0}\up{3.7}
& \cellcolor{moccolor}96.6\up{3.8}
& \cellcolor{moccolor}\textbf{90.2}\up{4.4}
& \cellcolor{moccolor}\textbf{96.4}\up{2.8}
& \cellcolor{moccolor}\textbf{90.1}\up{3.2}
& \cellcolor{moccolor}8.0
& \cellcolor{moccolor}61.3 \\

\Xhline{1.2pt}
\end{tabular}
}
\end{table*}

\begin{table}[t]
\centering
\caption{Comparison of aggregation schemes for MoC on Mistralv3-7B-Instruct and Qwen2.5-14B-Instruct.}
\vspace{0.5em}
\label{tab:aggregation_scheme}

\renewcommand\tabcolsep{4pt}
\renewcommand\arraystretch{1.2}

\resizebox{0.95\linewidth}{!}{
\begin{tabular}{c|l|c|cccccccc|c}
\Xhline{1.2pt}
\rowcolor{cyan!20}
\textbf{Model} & \textbf{Scheme}
& \textbf{\% Params}
& \textbf{BoolQ} & \textbf{PIQA} & \textbf{SIQA}
& \textbf{HellaSwag} & \textbf{WinoGrande}
& \textbf{ARC-e} & \textbf{ARC-c} & \textbf{OBQA}
& \textbf{Avg} \\
\Xhline{1.2pt}

\multirow{2}{*}{Mistralv3-7B}
& Mode
& 0.48
& 73.6
& 87.5
& 80.3
& 95.5
& 86.6
& \textbf{89.4}
& \textbf{79.6}
& \textbf{88.8}
& 85.2 \\

& \cellcolor{moccolor}\textbf{Mean}
& \cellcolor{moccolor}\textbf{0.48}
& \cellcolor{moccolor}\textbf{74.2}
& \cellcolor{moccolor}\textbf{88.9}
& \cellcolor{moccolor}\textbf{81.0}
& \cellcolor{moccolor}\textbf{95.6}
& \cellcolor{moccolor}\textbf{87.9}
& \cellcolor{moccolor}88.9
& \cellcolor{moccolor}79.5
& \cellcolor{moccolor}87.8
& \cellcolor{moccolor}\textbf{85.5} \\

\hline

\multirow{2}{*}{Qwen2.5-14B}
& Mode
& 0.43
& 73.8
& 89.8
& 81.9
& 96.5
& 89.2
& 95.5
& 88.7
& 95.4
& 88.9 \\

& \cellcolor{moccolor}\textbf{Mean}
& \cellcolor{moccolor}\textbf{0.43}
& \cellcolor{moccolor}\textbf{75.5}
& \cellcolor{moccolor}\textbf{91.8}
& \cellcolor{moccolor}\textbf{82.4}
& \cellcolor{moccolor}\textbf{96.8}
& \cellcolor{moccolor}\textbf{91.0}
& \cellcolor{moccolor}\textbf{96.6}
& \cellcolor{moccolor}\textbf{90.2}
& \cellcolor{moccolor}\textbf{96.4}
& \cellcolor{moccolor}\textbf{90.1} \\

\Xhline{1.2pt}
\end{tabular}
}
\end{table}

\begin{figure*}[t]
    \centering
    \begin{minipage}[t]{0.49\textwidth}
        \centering
        \includegraphics[width=\linewidth]{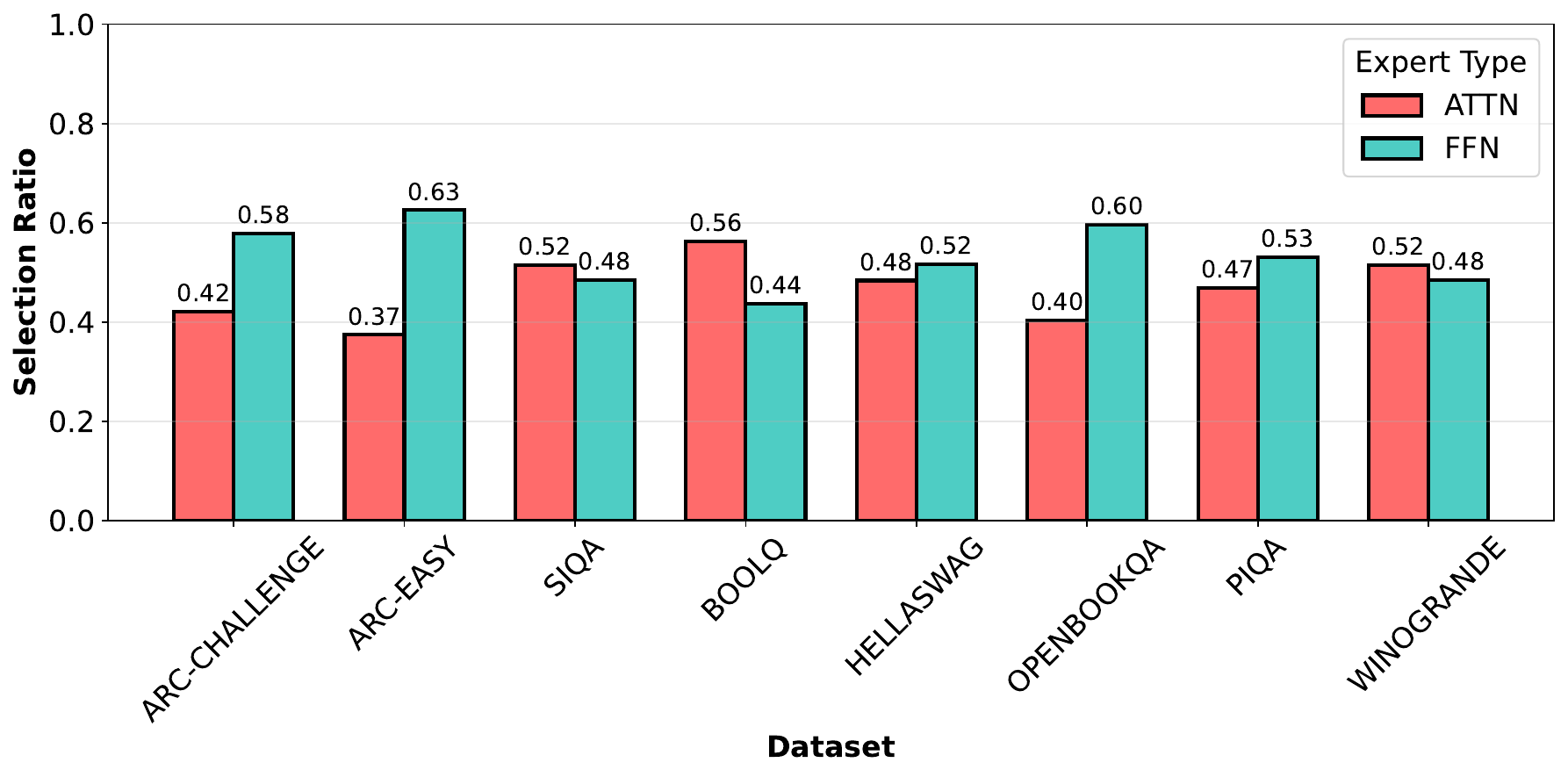}

        \vspace{0.3em}
        \small (a) Attention blocks
    \end{minipage}
    \hfill
    \begin{minipage}[t]{0.49\textwidth}
        \centering
        \includegraphics[width=\linewidth]{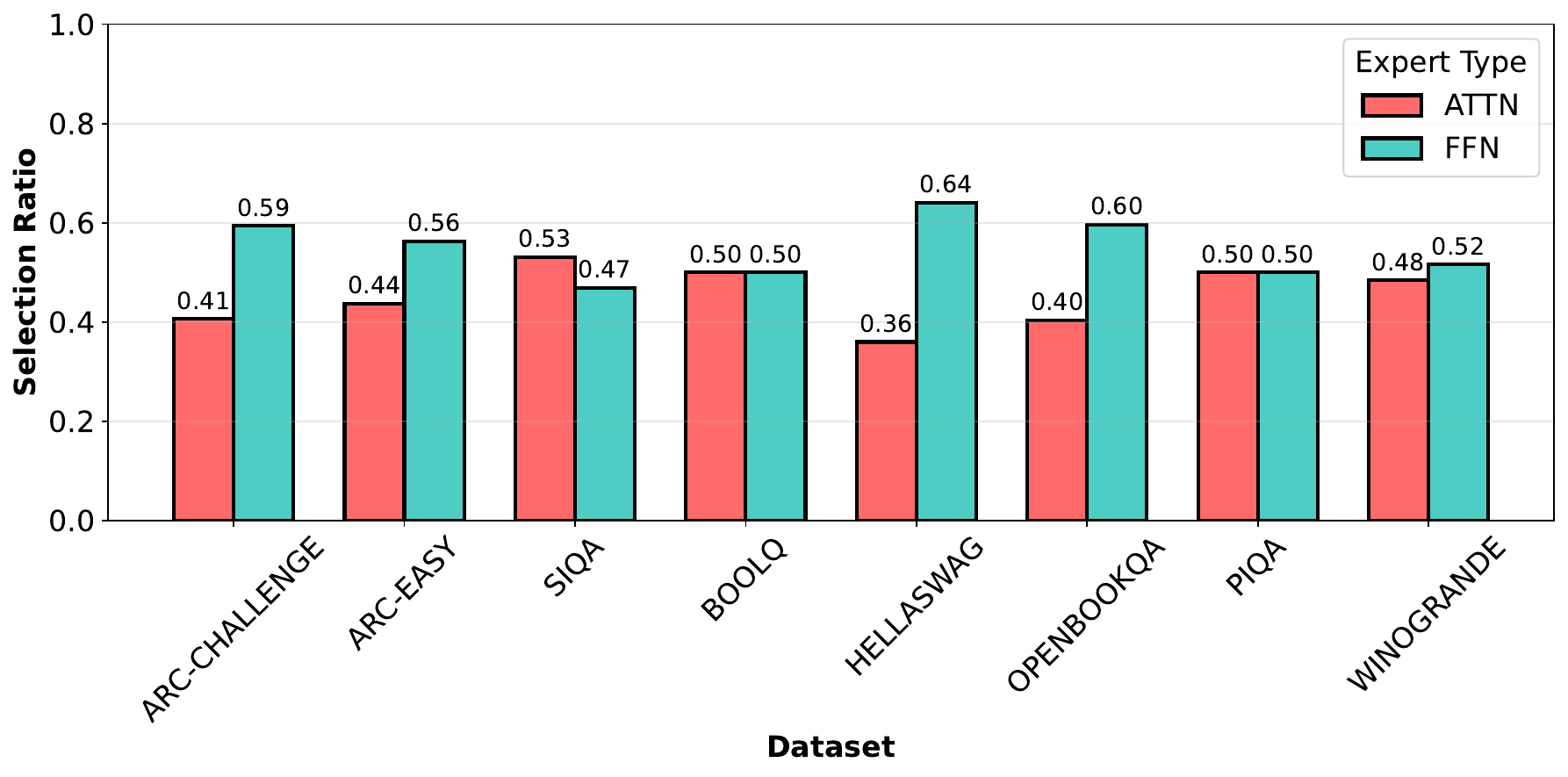}

        \vspace{0.3em}
        \small (b) FFN blocks
    \end{minipage}

    \caption{Cross-type routing behavior in LLaMA3-8B under the Top-$K=1$ setting, conditioned on the local control being from (a) Attention blocks or (b) FFN blocks.}
    \label{fig:crosstype_llama3_top1}
\end{figure*}

\begin{figure*}[t]
    \centering
    \begin{minipage}[t]{0.49\textwidth}
        \centering
        \includegraphics[width=\linewidth]{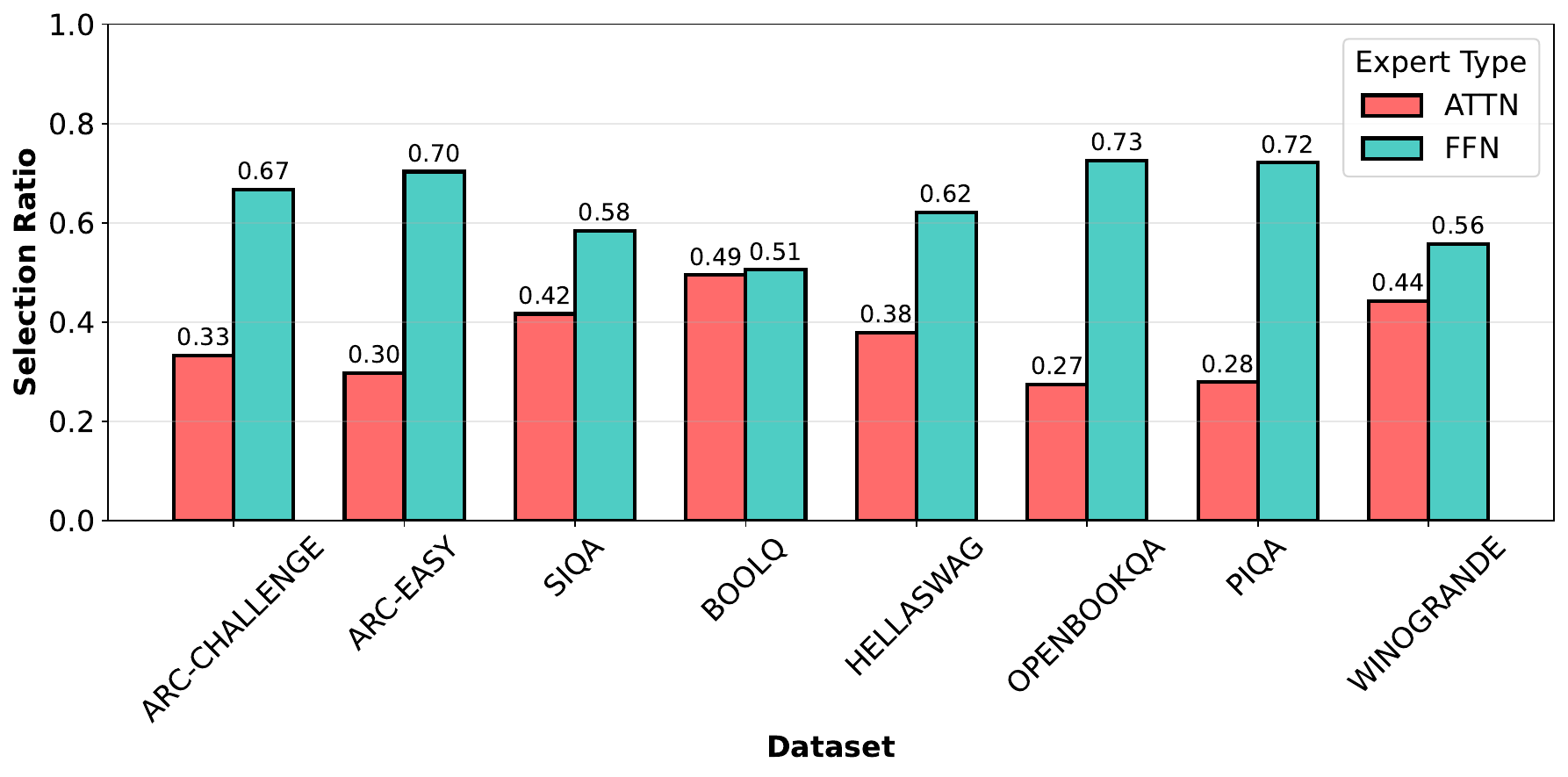}

        \vspace{0.3em}
        \small (a) Attention blocks
    \end{minipage}
    \hfill
    \begin{minipage}[t]{0.49\textwidth}
        \centering
        \includegraphics[width=\linewidth]{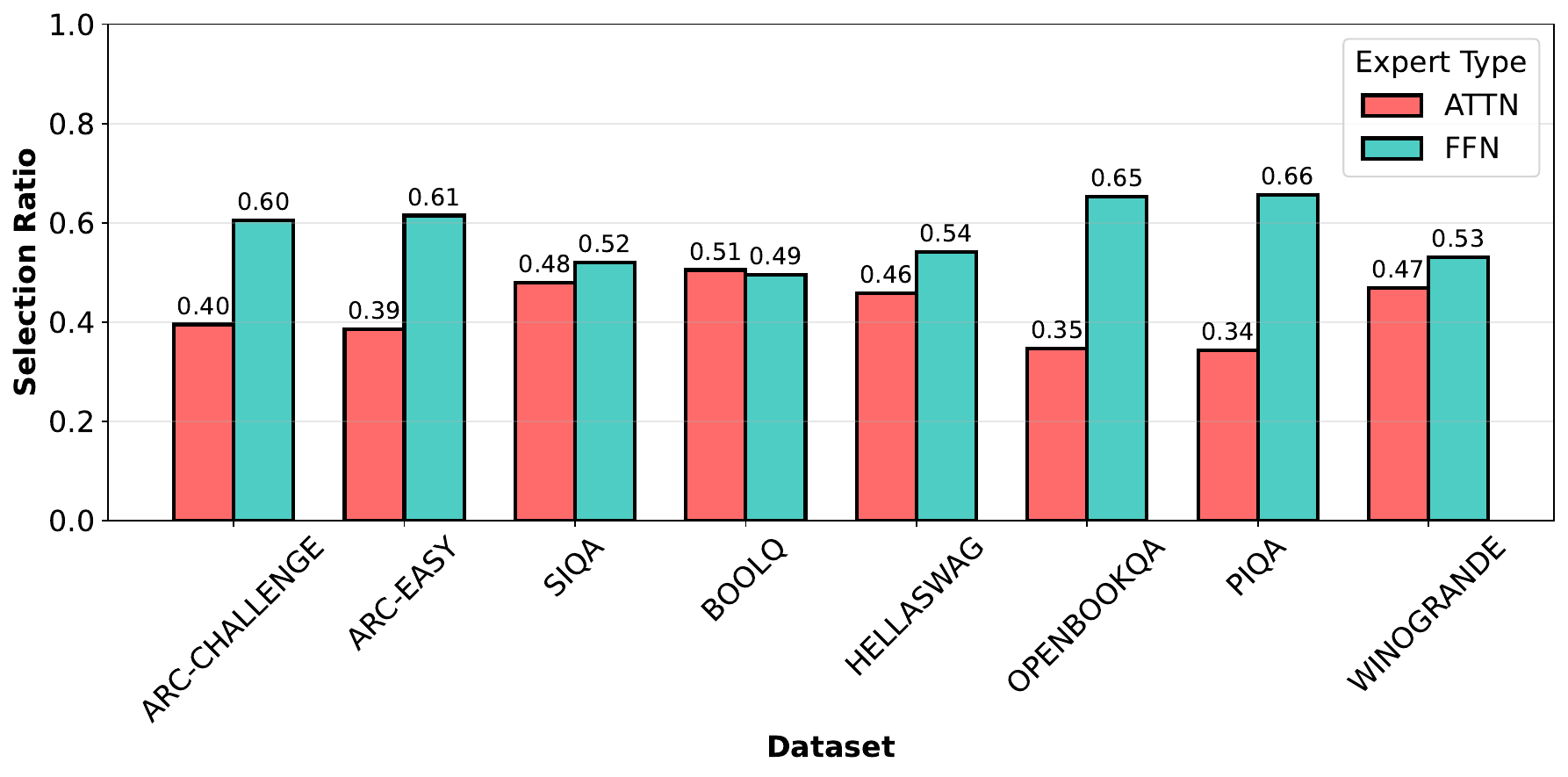}

        \vspace{0.3em}
        \small (b) FFN blocks
    \end{minipage}

    \caption{Cross-type routing behavior in LLaMA3-8B under the Top-$K=3$ setting, conditioned on the local control being from (a) Attention blocks or (b) FFN blocks.}
    \label{fig:crosstype_llama3_top3}
\end{figure*}

\begin{figure*}[t]
    \centering
    \begin{minipage}[t]{0.49\textwidth}
        \centering
        \includegraphics[width=\linewidth]{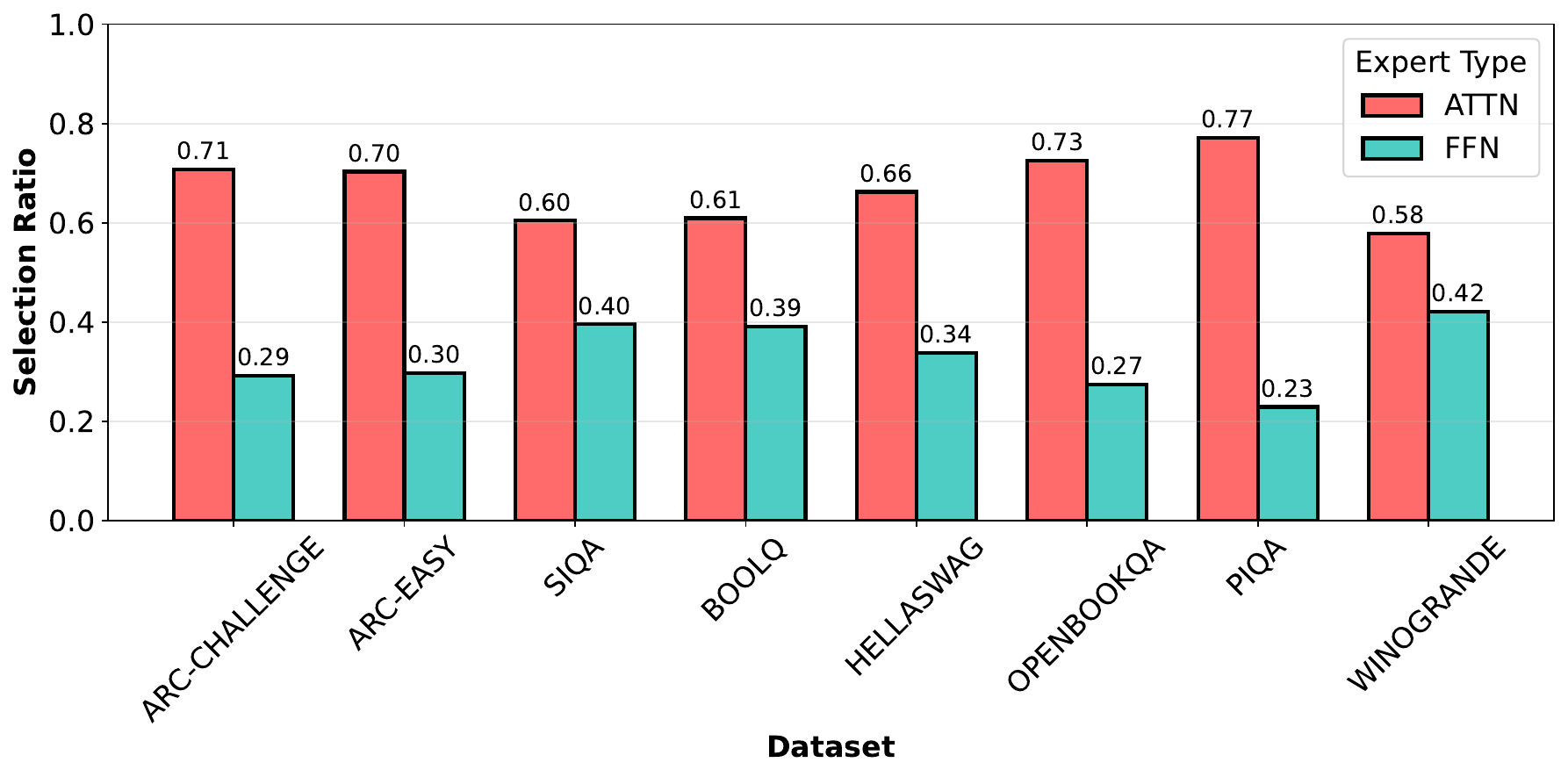}

        \vspace{0.3em}
        \small (a) Attention blocks
    \end{minipage}
    \hfill
    \begin{minipage}[t]{0.49\textwidth}
        \centering
        \includegraphics[width=\linewidth]{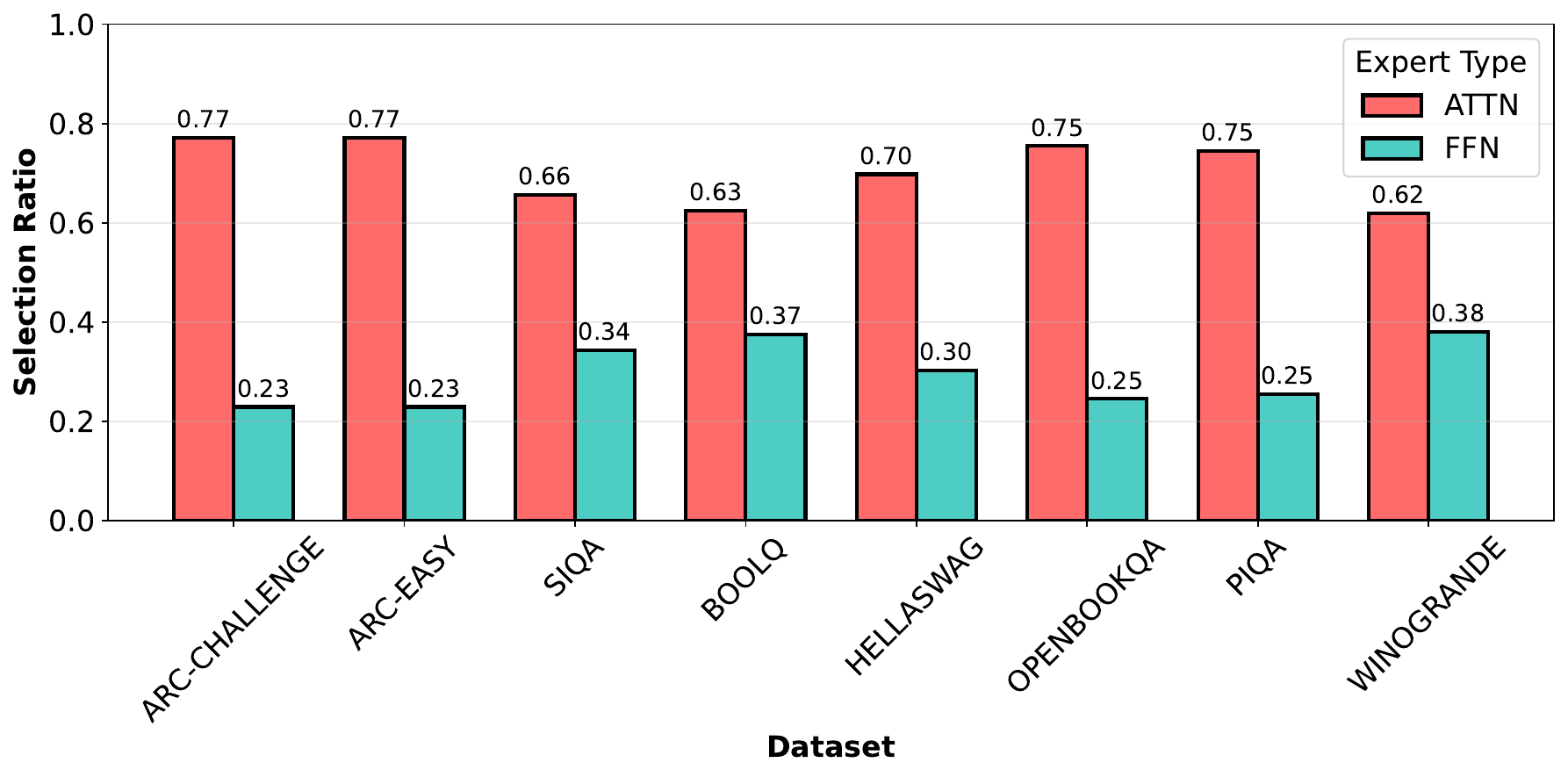}

        \vspace{0.3em}
        \small (b) FFN blocks
    \end{minipage}

    \caption{Cross-type routing behavior in LLaMA3-8B under the Top-$K=7$ setting, conditioned on the local control being from (a) Attention blocks or (b) FFN blocks.}
    \label{fig:crosstype_llama3_top7}
\end{figure*}

\paragraph{Inference Throughput.}
We measure the inference throughput of different LLMs across batch sizes, as shown in \cref{fig:inference_throughput}. Given 100 samples, we run batched inference and compute throughput as the number of generated tokens divided by the inference time for each batch. We then report the average throughput across all batches for each model.

\paragraph{Routing Behavior.} To better understand the routing behavior of MoC, we visualize the selected experts across layers for both Attention and FFN blocks under different routing budgets. The routing patterns show that, at a given timestep, MoC is not limited to the local control associated with the current layer. Instead, it often selects experts from non-local layers, suggesting that the model learns to dynamically aggregate cross-layer information for richer intermediate representations. We report the visualizations for LLaMA3-8B under Top-1, Top-3, and Top-7 routing, with similar routing patterns observed across other backbone models (see \cref{fig:crosstype_llama3_top1,fig:crosstype_llama3_top3,fig:crosstype_llama3_top7,fig:grid_attn_llama3_top1_all_datasets,fig:grid_ffn_llama3_top1_all_datasets,fig:grid_attn_llama3_top3_all_datasets,fig:grid_ffn_llama3_top3_all_datasets,fig:grid_attn_llama3_top7_all_datasets,fig:grid_ffn_llama3_top7_all_datasets}).

\begin{figure*}[t]
    \centering
    \includegraphics[width=\linewidth]{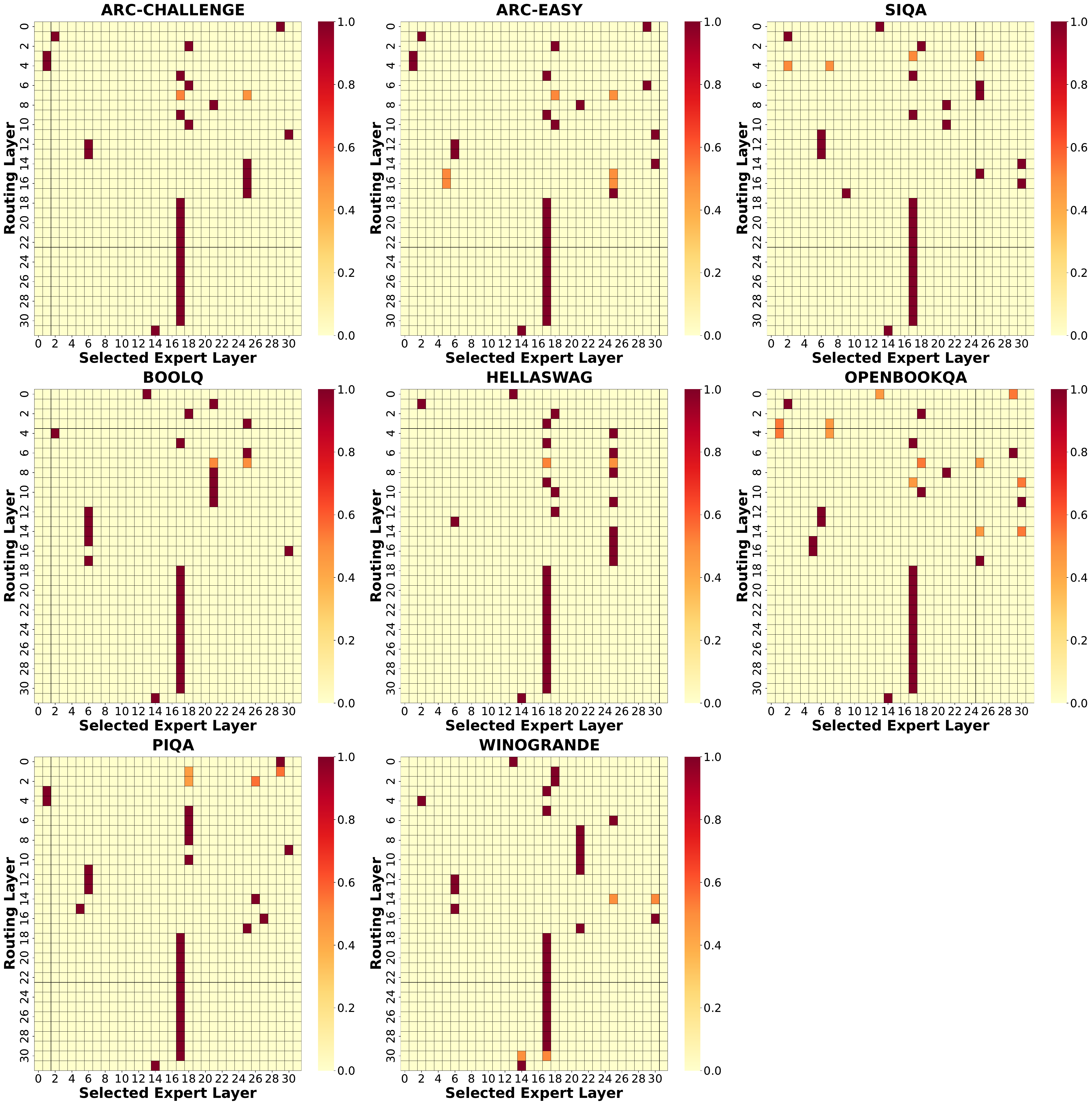}
    \caption{Routing behavior of MoC in LLaMA3-8B under Top-$K=1$ routing when selecting Attention-block controls across layers.}
    \label{fig:grid_attn_llama3_top1_all_datasets}
\end{figure*}

\begin{figure*}[t]
    \centering
    \includegraphics[width=\linewidth]{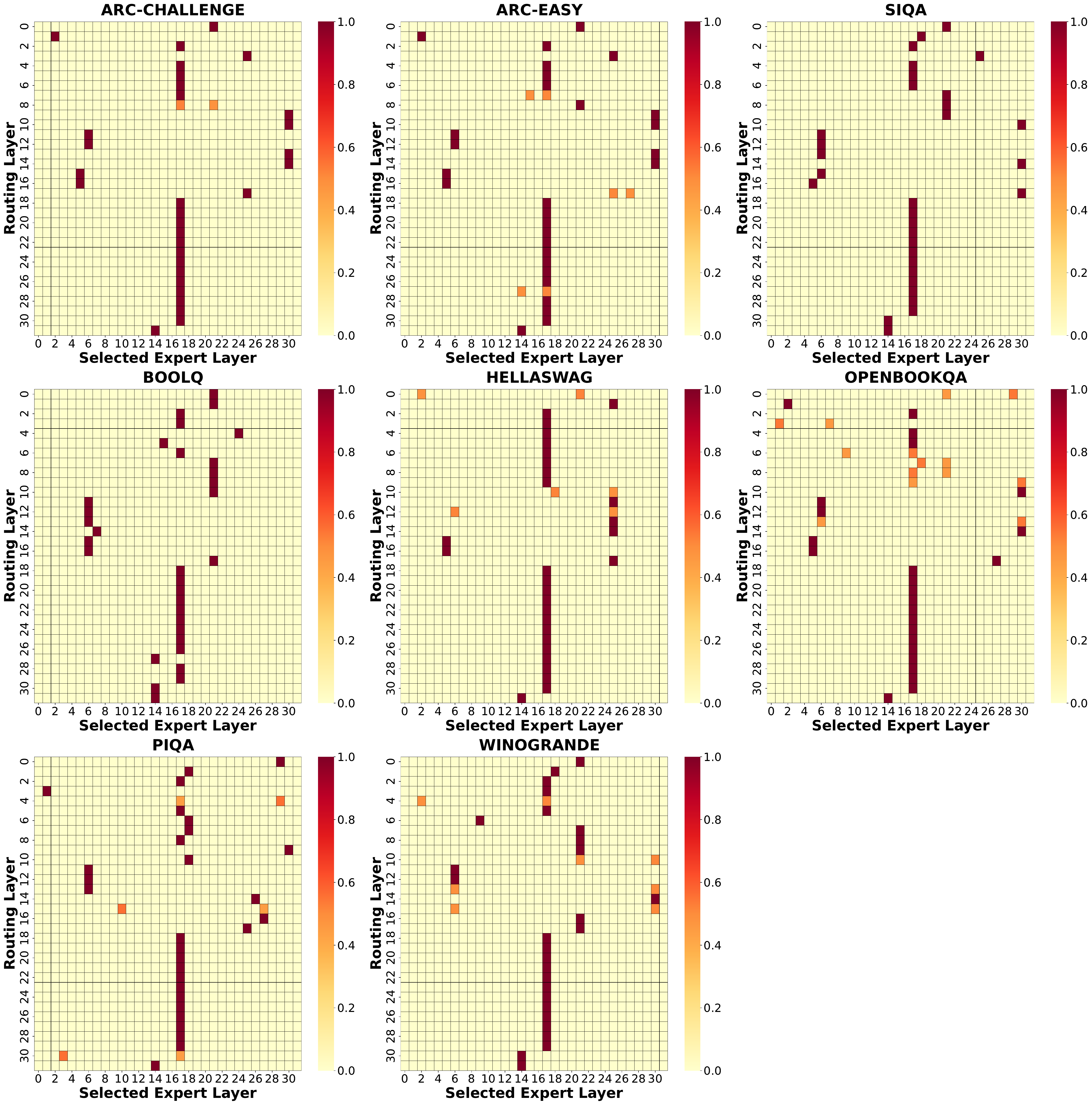}
    \caption{Routing behavior of MoC in LLaMA3-8B under Top-$K=1$ routing when selecting FFN-block controls across layers.}
    \label{fig:grid_ffn_llama3_top1_all_datasets}
\end{figure*}

\begin{figure*}[t]
    \centering
    \includegraphics[width=\linewidth]{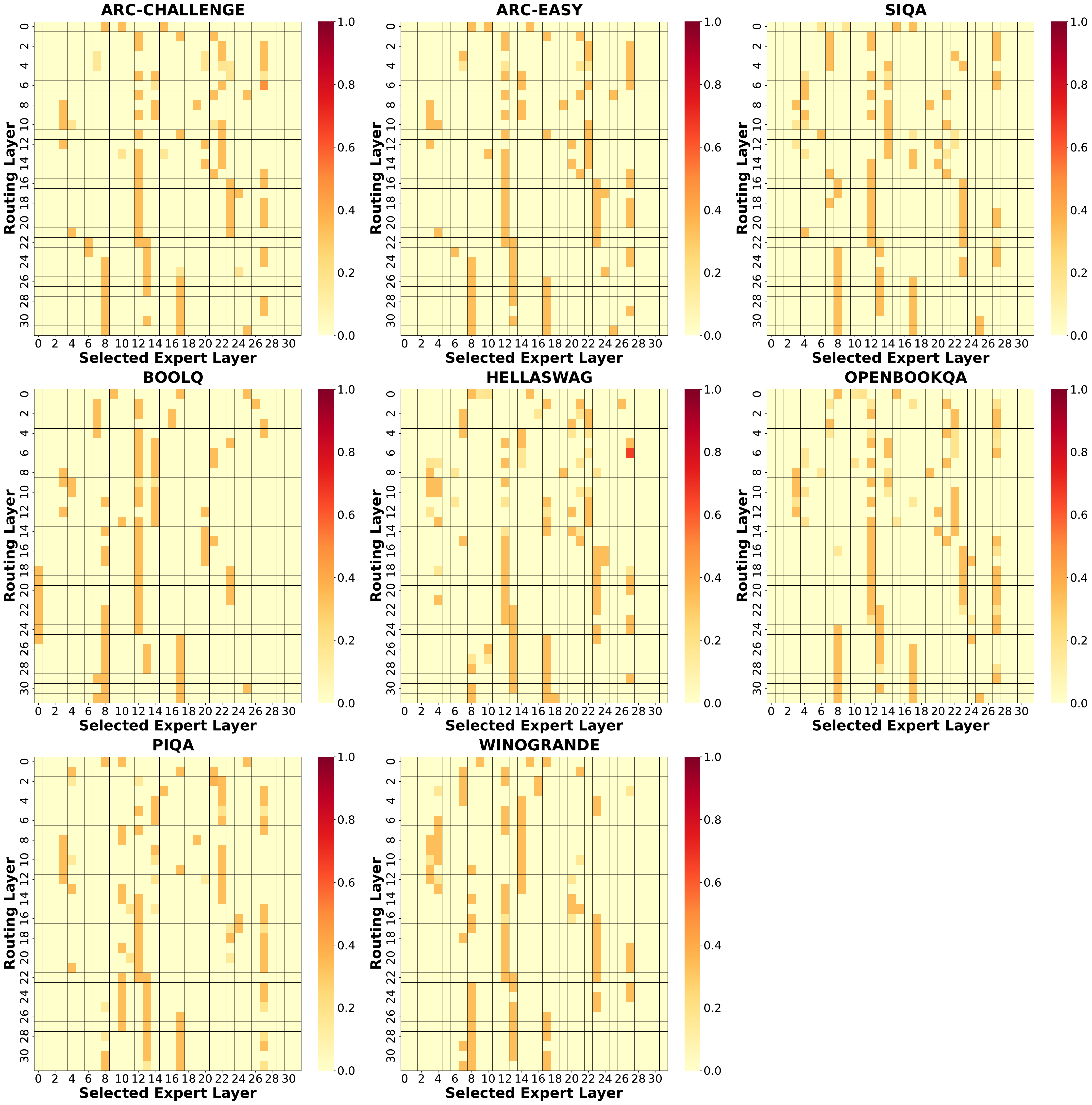}
    \caption{Routing behavior of MoC in LLaMA3-8B under Top-$K=3$ routing when selecting Attention-block controls across layers.}
    \label{fig:grid_attn_llama3_top3_all_datasets}
\end{figure*}

\begin{figure*}[t]
    \centering
    \includegraphics[width=\linewidth]{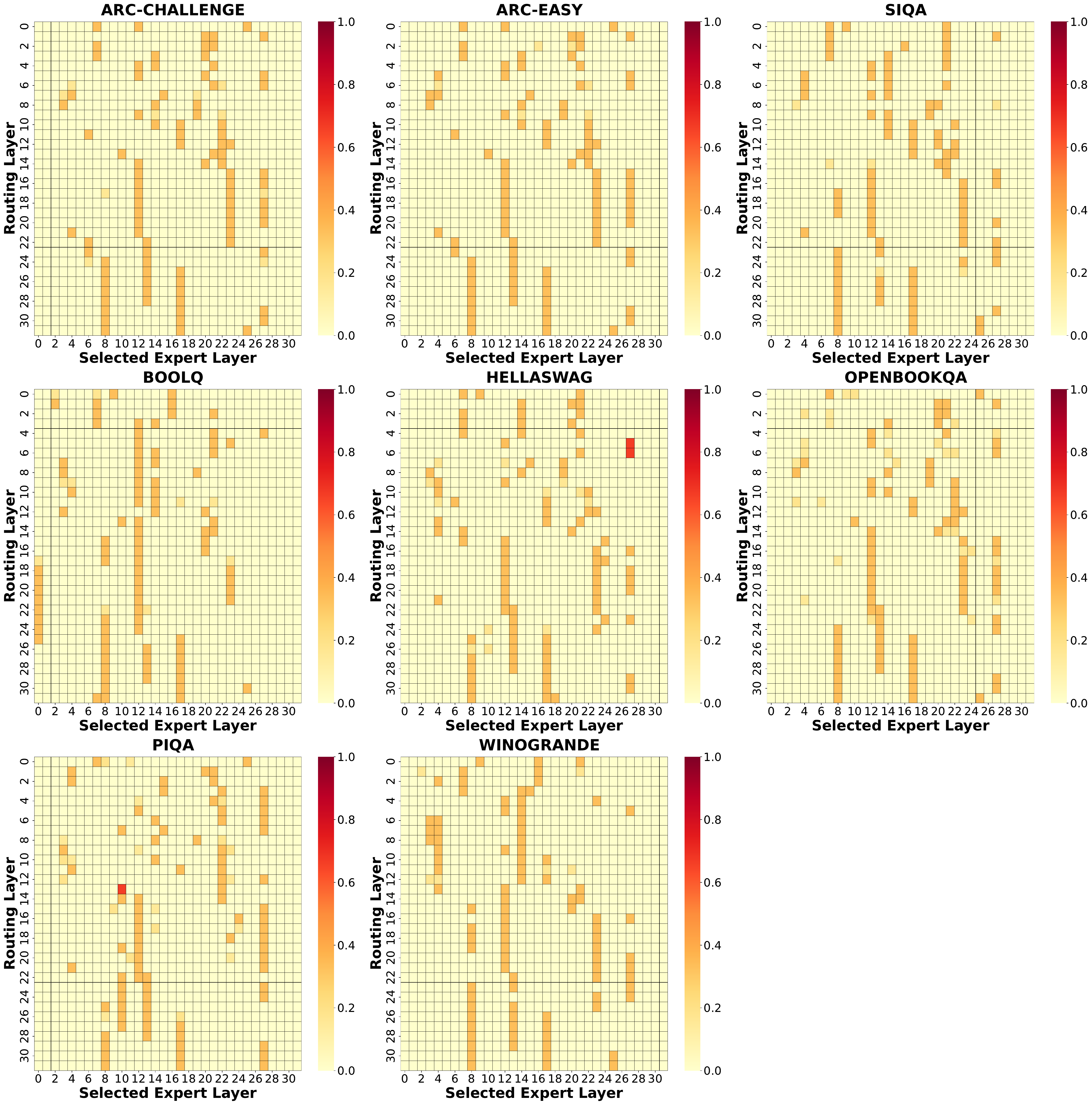}
    \caption{Routing behavior of MoC in LLaMA3-8B under Top-$K=3$ routing when selecting FFN-block controls across layers.}
    \label{fig:grid_ffn_llama3_top3_all_datasets}
\end{figure*}

\begin{figure*}[t]
    \centering
    \includegraphics[width=\linewidth]{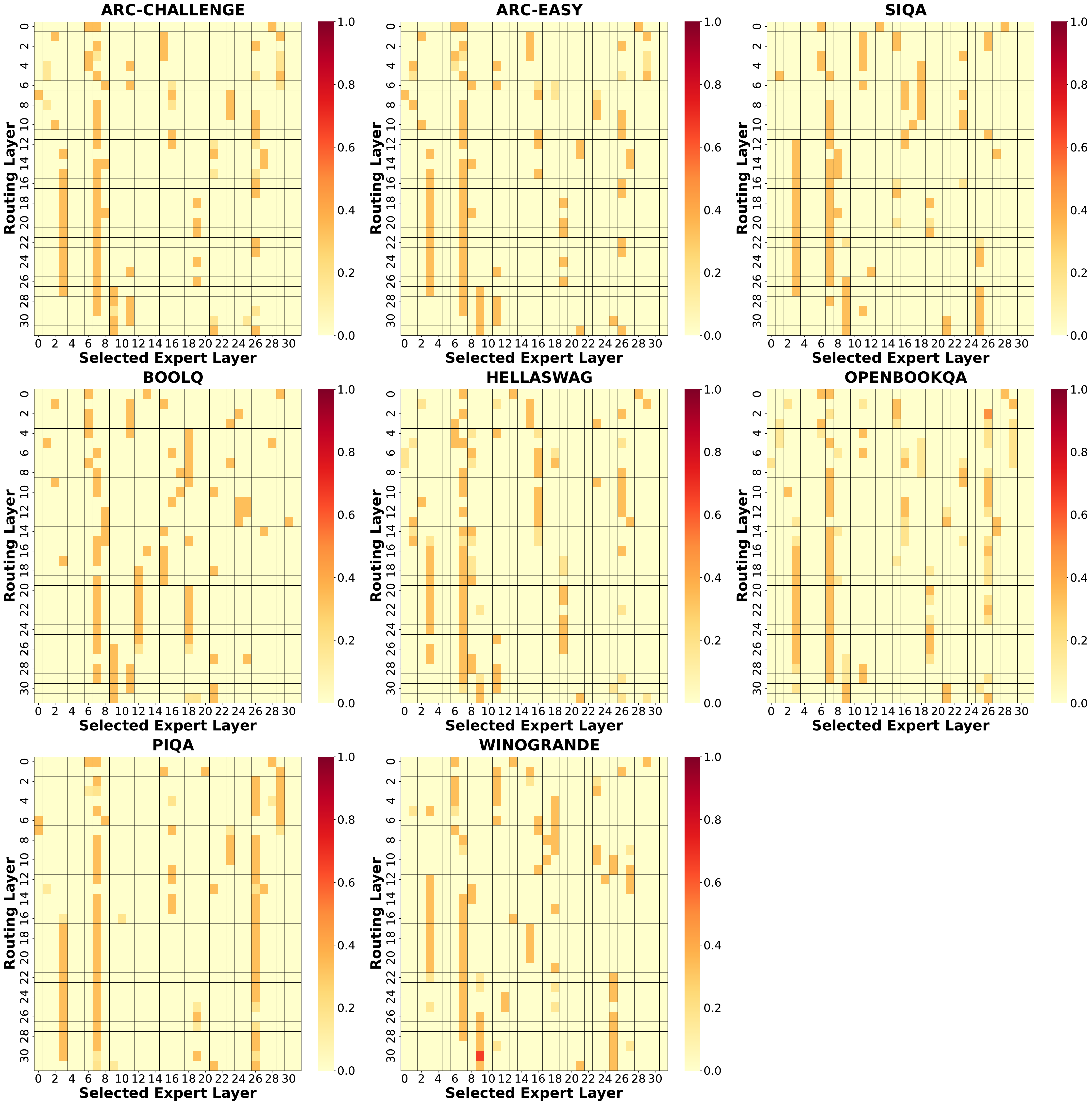}
    \caption{Routing behavior of MoC in LLaMA3-8B under Top-$K=7$ routing when selecting Attention-block controls across layers.}
    \label{fig:grid_attn_llama3_top7_all_datasets}
\end{figure*}

\begin{figure*}[t]
    \centering
    \includegraphics[width=\linewidth]{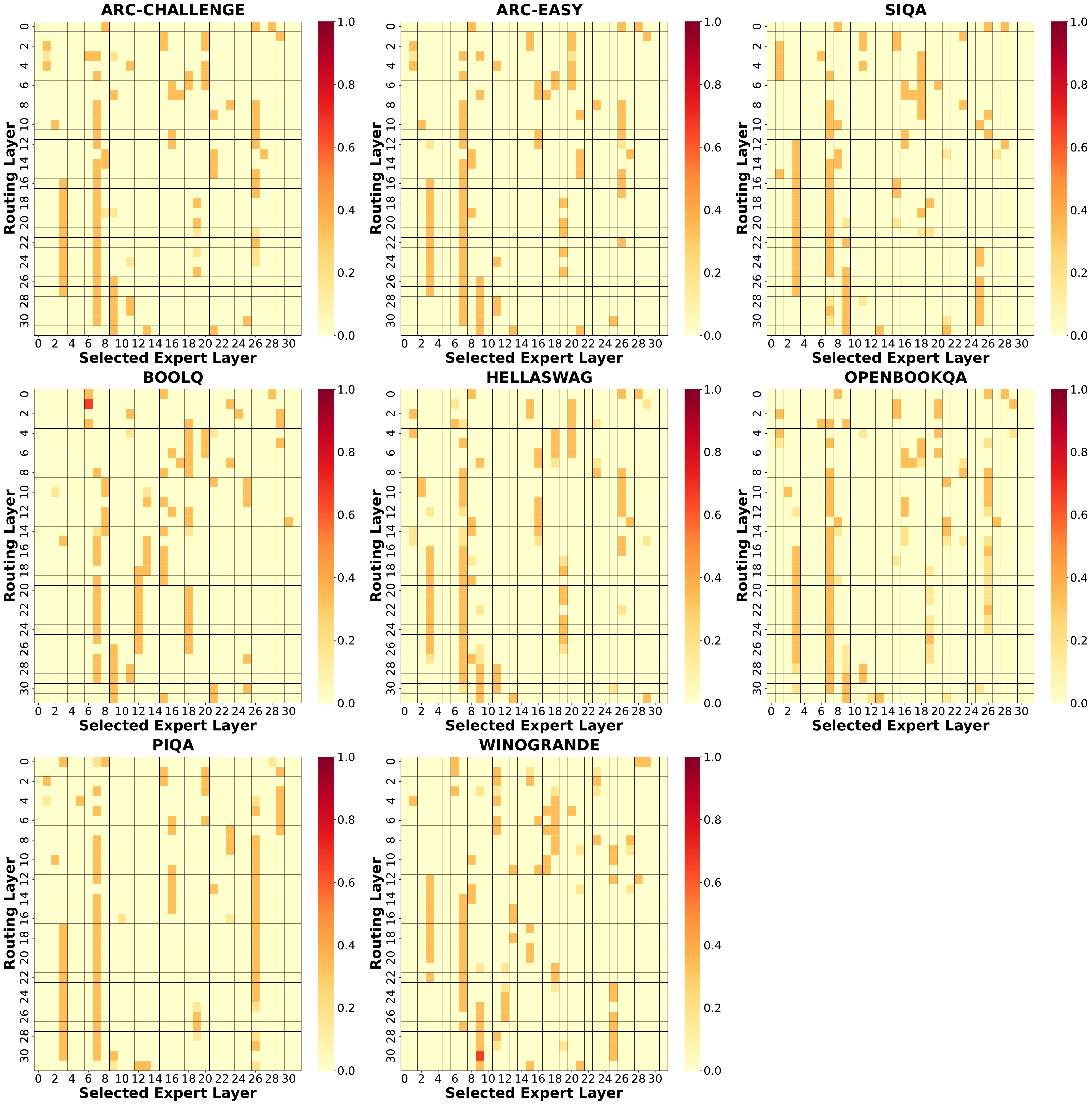}
    \caption{Routing behavior of MoC in LLaMA3-8B under Top-$K=7$ routing when selecting FFN-block controls across layers.}
    \label{fig:grid_ffn_llama3_top7_all_datasets}
\end{figure*}




\clearpage
\newpage
\section*{NeurIPS Paper Checklist}

\begin{enumerate}

\item {\bf Claims}
    \item[] Question: Do the main claims made in the abstract and introduction accurately reflect the paper's contributions and scope?
    \item[] Answer: \answerYes{} 
    \item[] Justification: The paper’s contributions and scope are reflected accurately in the abstract and introduction.

\item {\bf Limitations}
    \item[] Question: Does the paper discuss the limitations of the work performed by the authors?
    \item[] Answer: \answerYes{} 
    \item[] Justification: We have already discussed the limitations of our work in \cref{sec5:exp} and in Conclusion.

\item {\bf Theory assumptions and proofs}
    \item[] Question: For each theoretical result, does the paper provide the full set of assumptions and a complete (and correct) proof?
    \item[] Answer: \answerYes{} 
    \item[] Justification: The full set of assumptions and proofs is given in \cref{sec:4.1,sec4.3:theoretical_analysis}, as well as in Appendices.

    \item {\bf Experimental result reproducibility}
    \item[] Question: Does the paper fully disclose all the information needed to reproduce the main experimental results of the paper to the extent that it affects the main claims and/or conclusions of the paper (regardless of whether the code and data are provided or not)?
    \item[] Answer: \answerYes{} 
    \item[] Justification: All the experimental details for reproducibility are specified in \cref{sec5:exp} and Appendix \ref{appendix:exp_details}.

\item {\bf Open access to data and code}
    \item[] Question: Does the paper provide open access to the data and code, with sufficient instructions to faithfully reproduce the main experimental results, as described in supplemental material?
    \item[] Answer: \answerYes{} 
    \item[] Justification: We provide open access to the data and code, with sufficient instructions to
faithfully reproduce the main experimental results in supplemental material.

\item {\bf Experimental setting/details}
    \item[] Question: Does the paper specify all the training and test details (e.g., data splits, hyperparameters, how they were chosen, type of optimizer) necessary to understand the results?
    \item[] Answer: \answerYes{} 
    \item[] Justification: All the experimental details are specified in \cref{sec5:exp}, and Appendix \ref{appendix:exp_details}.

\item {\bf Experiment statistical significance}
    \item[] Question: Does the paper report error bars suitably and correctly defined or other appropriate information about the statistical significance of the experiments?
    \item[] Answer: \answerYes{} 
    \item[] Justification: Experimental results are tested multiple times to ensure stability and reliability
    (See \cref{sec5:exp}).

\item {\bf Experiments compute resources}
    \item[] Question: For each experiment, does the paper provide sufficient information on the computer resources (type of compute workers, memory, time of execution) needed to reproduce the experiments?
    \item[] Answer: \answerYes{} 
    \item[] Justification: We have provided sufficient information on the computer resources in Appendix \ref{appendix:exp_details}.
    
\item {\bf Code of ethics}
    \item[] Question: Does the research conducted in the paper conform, in every respect, with the NeurIPS Code of Ethics \url{https://neurips.cc/public/EthicsGuidelines}?
    \item[] Answer: \answerYes{} 
    \item[] Justification: We have read and followed all the NeurIPS Code of Ethics.

\item {\bf Broader impacts}
    \item[] Question: Does the paper discuss both potential positive societal impacts and negative societal impacts of the work performed?
    \item[] Answer: \answerNA{} 
    \item[] Justification: Given the primarily theoretical nature of this work, we do not anticipate direct positive or negative societal impacts. The experiments are conducted solely to empirically support and validate the theoretical findings.
    
\item {\bf Safeguards}
    \item[] Question: Does the paper describe safeguards that have been put in place for responsible release of data or models that have a high risk for misuse (e.g., pre-trained language models, image generators, or scraped datasets)?
    \item[] Answer: \answerNA{} 
    \item[] Justification: Our paper does not release any data or models that have a high risk for misuse.

\item {\bf Licenses for existing assets}
    \item[] Question: Are the creators or original owners of assets (e.g., code, data, models), used in the paper, properly credited and are the license and terms of use explicitly mentioned and properly respected?
    \item[] Answer: \answerYes{} 
    \item[] Justification: The data and models used in the paper are properly credited. See \cref{sec5:exp} as well as Appendix \ref{appendix:exp_details} for further details.

\item {\bf New assets}
    \item[] Question: Are new assets introduced in the paper well documented and is the documentation provided alongside the assets?
    \item[] Answer: \answerNA{} 
    \item[] Justification: Our paper does not release new assets.

\item {\bf Crowdsourcing and research with human subjects}
    \item[] Question: For crowdsourcing experiments and research with human subjects, does the paper include the full text of instructions given to participants and screenshots, if applicable, as well as details about compensation (if any)? 
    \item[] Answer: \answerNA{} 
    \item[] Justification: Our paper does not involve crowdsourcing nor research with human subjects.

\item {\bf Institutional review board (IRB) approvals or equivalent for research with human subjects}
    \item[] Question: Does the paper describe potential risks incurred by study participants, whether such risks were disclosed to the subjects, and whether Institutional Review Board (IRB) approvals (or an equivalent approval/review based on the requirements of your country or institution) were obtained?
    \item[] Answer: \answerNA{} 
    \item[] Justification: Our paper does not involve crowdsourcing nor research with human subjects.

\item {\bf Declaration of LLM usage}
    \item[] Question: Does the paper describe the usage of LLMs if it is an important, original, or non-standard component of the core methods in this research? Note that if the LLM is used only for writing, editing, or formatting purposes and does \emph{not} impact the core methodology, scientific rigor, or originality of the research, declaration is not required.
    \item[] Answer: \answerNA{} 
    \item[] Justification: The core contribution of the method does not involve LLMs.

\end{enumerate}

\end{document}